%% file: DONS.tex
\documentclass[11pt]{article}
\input{tex_header}

\usepackage{algorithm}
\usepackage{algorithmic}
\newcommand{\algcomment}[1]{\textcolor{blue!70!black}{\footnotesize{\texttt{\text{//
					#1}}}}}

\usepackage{comment}
\usepackage{bm}

\usepackage{bbding}
\usepackage{booktabs}
\usepackage{footnote}
\makesavenoteenv{tabular}
\makesavenoteenv{table}
\let\norm\undefined

\let\ceil\undefined
\usepackage{notation}
\usepackage{bm}
\mathtoolsset{showonlyrefs}
\usepackage{multirow}
\usepackage{microtype}

\newcommand{\cS}{\mathcal{S}}

\DeclareRobustCommand{\VAN}[3]{#2} 

\usepackage{colortbl}

\makeatletter
\g@addto@macro\bfseries{\boldmath}
\makeatother

\newcommand*\accentfontxheight[1]{%
	\fontdimen5\ifx#1\displaystyle
	\textfont
	\else\ifx#1\textstyle
	\textfont
	\else\ifx#1\scriptstyle
	\scriptfont
	\else
	\scriptscriptfont
	\fi\fi\fi3
}
\let\wtilde\undefined

\makeatletter
\newcommand*\wtilde[1]{\mathpalette\wthelpers{#1}}

\newcommand*\wthelpers[2]{%
	\hbox{\dimen@\accentfontxheight#1%
		\accentfontxheight#11.2\dimen@
		$\m@th#1\widetilde{#2}$%
		\accentfontxheight#1\dimen@
	}%
}

\begin{document}
	
%
	
	\title{Damped Online Newton Step for Portfolio Selection}
	\author{{\bf Zakaria Mhammedi}    \\
	{\bf Alexander Rakhlin}  \\ Massachusetts Institute of Technology \\ \texttt{\{mhammedi, rakhlin\}@mit.edu}}

	\maketitle
	
	\begin{abstract}
		We revisit the classic online portfolio selection problem, where at each round a learner selects a distribution over a set of portfolios to allocate its wealth. It is known that for this problem a logarithmic regret with respect to Cover's loss is achievable using the Universal Portfolio Selection algorithm, for example. However, all existing algorithms that achieve a logarithmic regret for this problem have per-round time and space complexities that scale polynomially with the total number of rounds, making them impractical. In this paper, we build on the recent work by Haipeng et al. 2018 and present the first practical online portfolio selection algorithm with a logarithmic regret and whose per-round time and space complexities depend only logarithmically on the horizon. Behind our approach are two key technical novelties of independent interest. We first show that the Damped Online Newton steps can approximate mirror descent iterates well, even when dealing with time-varying regularizers. Second, we present a new meta-algorithm that achieves an adaptive logarithmic regret (i.e.~a logarithmic regret on any sub-interval) for mixable losses.
	\end{abstract}

	\section{Introduction}
	In this paper, we consider the problem of online portfolio selection where, at each round $t$, a learner chooses a distribution $\p_t\in \Delta_d$ over a fixed set of $d$ portfolios. Then, the environment reveals a return vector $\r_t \in \reals_{\geq 0}^d$, and the learner suffers a loss $ \ell_t(\p_t)\coloneqq - \ln \inner{\p_t}{\r_t}$. The goal of the learner is to minimize the regret $\cR_T(\u)\coloneqq\sum_{t=1}^T( \ell_t(\p_t)-\ell_t(\u))$ after $T\ge 1$ rounds, which is the difference between the cumulative loss of the learner minus that of any distribution $\u$ over portfolios. For this problem, it is known that Cover's Universal Portfolio Algorithm (UPA) \cite{cover1991universal} guarantees the optimal $O(d\ln T)$ regret bound. One implication of this is that if a distribution $\u$ has an exponential return growth rate with constant $\lambda>0$, i.e.~$\prod_{t\in[T]} \inner{\u}{\r_t} \propto e^{\lambda T}$, then the total return of UPA also has an exponential growth rate with constant at least $\Omega(\lambda/d)$. 
	
	The main shortcoming of the UPA is that the expression of its outputs involves multi-variate integrals that make its implementation impractical. One way of approximating these integrals is via log-concave sampling as done by \cite{kalai2002efficient}. The algorithm of the latter has a computational complexity of order $O(d^4 T^{15} )$, measured after $T$ rounds. Even though this computational complexity can be reduced using more modern log-concave sampling methods (see e.g.~\cite{narayanan2017efficient,bubeck2018sampling}), it remains a large polynomial of $T$, making these approaches impractical.
	
	It is possible to use other more efficient online learning algorithms for the portfolio selection problem. Algorithms such as Online Gradient Descent \cite{zinkevich2003online}, Online Newton Step \cite{hazan2007logarithmic}, and Exponentiated Gradients \cite{helmbold1998line} all have regret bounds that scale with the largest observed gradient norm $G$. One way to ensure that the gradients are bounded in the online portfolio setting is to mix the outputs of such algorithms with a small amount of uniform distribution. This approach leads to a regret bound of order $d\sqrt{T\ln d}$ in the best case (which is not logarithmic in $T$), even after optimizing for the amount of uniform distribution used. The Soft-Bayes algorithm \cite{orseau2017soft} provides a $\sqrt{d}$ improvement over this regret bound. Finally, \cite{agarwal2005efficient} showed that the Follow-the-Regularized-Leader (FTRL) algorithm achieves a regret bound of order $O(G^2 d \ln (dT))$, albeit it was conjectured by \cite{van2020open} that the dependence in $G$ (the largest gradient norm) may be merely an artifact of the analysis. Tab.~\ref{tab:rates} compares the regret bounds and computational complexities of the different algorithms mentioned here. 
	
	Among known algorithms that achieve a logarithmic regret in the online portfolio setting, \adabar{} \cite{luo2018efficient} is the best in terms of computational cost (see Tab.~\ref{tab:rates}). \adabar{} consists of I) a base algorithm that is essentially mirror descent with a log-barrier plus a quadratic regularizer with a parameter $\beta$; and II) a meta-algorithm that implements a clever restart scheme to learn the parameter $\beta$ and achieve a logarithmic regret. The algorithmic idea behind \adabar{} can be traced back to the problem of combining bandit algorithms \cite{agarwal2017corralling,wei2018more}, where the use of a non-decreasing learning rate schedule is used to extract crucial negative terms in the regret analysis of mirror descent (see \S\ref{sec:sketchit}). 
	
	The main drawback of \adabar{} is that its time [resp.~space] complexity is quadratic [resp.~linear] in the total number of rounds (see Table \ref{tab:rates}). Though \adabar{} has a substantially better computational complexity compared to approaches based on log-concave sampling, it is still not a practical algorithm when the horizon is large. The main reason for the quadratic time complexity is the restarts of \adabar, which require computing the regularized leader at each round. \cite{luo2018efficient} posed the question of whether there exists an algorithm that improves on either the regret or the computational complexity of Ada-BARRONS without hurting the other.
	
	\begin{table}
		\fontsize{10}{10}\selectfont
		\centering
		\begin{tabular}{ccccc}
			\hline
			Algorithm & Regret & Run-Time  & Space Comp. & References  \\ 
			\hline 
			Universal Portfolio & $d\ln T$ & $d^4T^{15}$ &  $dT$ & \cite{cover1991universal,kalai2002efficient} \\  
			\hline
			ONS & $Gd\ln T$ & $d^{3.5}T$ & $d$ & \cite{hazan2007logarithmic} \\
			\hline
			FTRL~ & $G^2 d\ln (dT)$ & $d^{2.5}T^2$ & $dT$ & \cite{agarwal2005efficient} \\
			\hline
			EG~ & $G\sqrt{T\ln d}$ & $dT$ & $d$ & \cite{helmbold1998line} \\
			\hline
			Soft-Bayes & $ \sqrt{dT\ln d}$ & $dT$ & $d$ & \cite{orseau2017soft}  \\
			\hline
			\adabar  & $ d^2\ln^4 T $ & $d^{2.5}T^2+d^{3.5}T $ & $d T$ & \cite{luo2018efficient} \\
			\hline 
			$\mathsf{AdaMix}+\mathsf{DONS}$	&   $ d^2\ln^5 T $ & $d^3 T\ln^2 T$ & $d \ln^2 T $  & ({\bf this work}---Thm.~\ref{thm:mainthm})\\  
			\hline 
			\hline
		\end{tabular}
		\label{tab:rates}
		\caption{Result comparison.}
	\end{table}
	
	\paragraph{Contributions.} 
	We answer the above question in the positive by presenting an online algorithm for portfolio selection with a logarithmic regret and that has near constant per-round time and space complexities. Behind our solution are two techniques of independent interest in online learning. We first show that one can use the \emph{damped Newton steps} \cite{nesterov2018lectures} to approximate the mirror descent iterates in Ada-BARRONS without sacrificing the logarithmic regret. Here, existing results due to \cite{abernethy2012interior} do not apply (due to time-varying regularizers in the mirror descent objective---see \S\ref{sec:barrons}). Even if they did, they would lead a suboptimal $O(\sqrt{T})$ regret, and so a new analysis is needed, which we provide. Using online damped Newton steps confers a $O(\sqrt{d})$ improvement in the computational cost.
	
	The second and crucial tool we use is a new meta-algorithm that achieves an adaptive \cite{hazan2007adaptive} logarithmic regret for mixable losses; that is, an algorithm that achieves a logarithmic regret on any interval $I\subseteq [T]$, whenever the losses are mixable. Thanks to a novel analysis, we show that using such a meta-algorithm removes the need for computing the regularized leader, which is required by \adabar. This further improves the time and space complexities by $\wtilde O(T)$, leading to our final algorithm that has $O(d^3 T \ln^2 T)$ and $O(d\ln^2 T)$ total time and space complexities, respectively.
	
	The techniques we develop are transferable to another prominent online learning problem; that of learning linear models with the log-loss \cite[Section 6]{rakhlin2015sequential}. 
	
	\paragraph{Outline.}
	In \S\ref{sec:prelim}, we introduce the notation and definitions we need. We also include some results on self-concordance that we require in our analysis. In \S\ref{sec:sketchit}, we describe the \adabar{} algorithm in more detail and highlight the challenges involved in the design of an efficient alternative. There, we also outline our solution and give a sketch of why it works. Finally, in \S\ref{sec:fulldetails}, we present the full details of our algorithm and its guarantee. The proofs are differed to the appendix.

	\section{Preliminaries}
	\label{sec:prelim}
	We define the set $
	\cC_{d-1} \coloneqq \{ \u \in \reals^{d-1}_{\geq 0}\colon \ \ \inner{\mathbf{1}}{\u} \leq 1 \}.$ Throughout, for any $\v \in \cC_{d-1}$, we denote
	\begin{align}
		\v' \coloneqq (1-1/T) \v + \mathbf{1}/(dT),\quad \text{and}\quad
		\bar \v \coloneqq \e_d + J^\top \v, \ \ \text{where}  \ \   J \coloneqq \begin{bmatrix}I & - \mathbf{1}\end{bmatrix}.
	\end{align}
	We may combine the notation and write $\bar \v'\coloneqq (1-1/T) \bar \v + \mathbf{1}/(dT)$ and $\v''\coloneqq (1-1/T)  \v' + \mathbf{1}/(dT)$. We will be working with Cover's loss $\ell_t$, which for a return vector $\r_t\in\reals_{\geq 0}^d$, is given by
	\begin{gather}
		\forall \u\in \cC_{d-1}, \quad	\ell_t(\u) \coloneqq - \log \inner{\r_t}{\bar \u}. \label{eq:cover} 
	\end{gather}	
	Our goal is to design an efficient algorithm whose outputs $(\u_t)$ are such that the regret
	\begin{align}
		\text{Regret}_T(\u) \coloneqq \sum_{t=1}^T (\ell_t(\u_t) - \ell_t(\u)) =  \ln \frac{\prod_{t=1}^T\inner{\r_t}{\bar \u}}{\prod_{t=1}^T\inner{\r_t}{\bar\u_t}}
	\end{align}
	against any comparator $\u\in \cC_{d-1}$ is bounded by a poly-logarithmic factor in $T$. Since the regret is invariant to the scale of $(\r_t)$, we may assume without loss of generality that $(\r_t)\subset [0,1]^d$. The next lemma (taken from \cite[Lemma 10]{luo2018efficient}), implies that a regret against a comparator $\u\in \cC_{d-1}$ is bounded by the regret against $\u'$ up to an additive factor---this will be useful throughout:
	\begin{lemma}
		\label{lem:additive}
		For any $\u \in \cC_{d-1}$ and $\u' =(1-\frac{1}{T}) \u + \frac{1}{dT}\mathbf{1}$, we have $\sum_{t=1}^T \ell_t(\u)\leq \sum_{t=1}^T \ell_t(\u')+2.$
	\end{lemma}
	\paragraph{Self-Concordant Functions.}
	We now present some results on self-concordant functions that we will make heavy use of in the proofs of our results. We start by the definition of a self-concordant function. For the rest of this section, we let $\cK$ be convex compact set with non-empty interior $\inte \cK$. For a twice [resp.~thrice] differentiable function, we let $\nabla^2 f(\u)$ [resp.~$\nabla^3 f(\u)$] be the Hessian [resp.~third derivative tensor] of $f$ at $\u$. 
	\begin{definition}
		A convex function $f\colon \inte \cK \rightarrow \reals$ is called \emph{self-concordant} with constant $M_f\geq 0$, if $f$ is $C^3$ and satisfies {\bf I)} $f(\x_k)\to +\infty$ for $\x_k \to \x\in \partial \cK${\em ;} and {\bf II)}
		\begin{align}
			\forall \x\in  \inte\cK, \forall \u \in \reals^d, \quad |\nabla^3f(\x)[\u,\u,\u]| \leq 2 M_f \|\u\|^3_{\nabla^2 f(\x)}.
		\end{align}
	\end{definition} 
	Note that by definition, if $f$ is self-concordant with constant $M_f\geq 0$ it is also self-concordant with any constant $M\geq M_f$. Another property that we will use is that if $f_1$ and $f_2$ are self-concordant functions with constants $M_1$ and $M_2$, respectively, then for any $\alpha,\beta>0$, the function $\alpha f_1 + \beta f_2$ is self-concordant with constant $\frac{M_1}{\sqrt{\alpha}} \wedge \frac{M_2}{\sqrt{\beta}}$ \cite[Theorem~5.1.1]{nesterov2018lectures}.
	
	For a self-concordant function $f$ and $\x\in \dom f$, the quantity $\lambda(\x, f)\coloneqq \|\nabla f(\x)\|_{\nabla^{-2}f(\x)}$, known as the \emph{Newton decrement}, will be instrumental in our proofs. The following two lemmas contain properties of the Newton decrement and Hessians of self-concordant functions, which we will use repeatedly throughout (see e.g. \cite{nemirovski2008interior,nesterov2018lectures}).
	\begin{lemma}
		\label{lem:properties}
		Let $f\colon \inte \cK\rightarrow \reals$ be a self-concordant function with constant $M_f\geq 1$. Further, let $\x\in \inte \cK$ and $\x_f\in \argmin_{\x\in \cK} f(\x)$. Then, {\bf I)} whenever $\lambda(\x,f)<1/M_f$, we have 
		\begin{align}
			\|\x -\x_f\|_{\nabla^2 f(\x_f)} 	\vee 	\|\x -\x_f\|_{\nabla^2 f(\x)} \leq {\lambda(\x,f)}/({1-M_f \lambda (\x,f)});  \label{eq:first1}
		\end{align}
		and {\bf II)} for any $M\geq M_f$, the \emph{damped Newton step} $\x_+\coloneqq \x - \frac{1}{1+M \lambda(\x,f)} \nabla^{-2}f(\x)\nabla f(\x)$ satisfies $\x_+\in \cK$ and 
		$ \lambda(\x_+,f)\leq M \lambda(\x,f)^2 ( 1 + (1+M \lambda(\x,f))^{-1}).$
	\end{lemma}
	\begin{lemma}
		\label{lem:hessians}
		Let $f\colon \inte \cK\rightarrow \reals$ be a self-concordant function with constant $M_f$ and $\x \in \inte \cK$. Then, for any $\y$ such that $r\coloneqq \|\y - \x\|_{\nabla^2 f(\x)} < 1/M_f$, we have 
		\begin{align}
			(1-M_f r)^{2} \nabla^2 f(\y) \preceq \nabla^2 f(\x) \preceq (1-M_f r)^{-2}  \nabla^2 f(\x).
		\end{align}
	\end{lemma}
	A consequence of the latter lemma is the following useful result whose proof is in Appendix \ref{sec:cproofs}:
	\begin{lemma}
		\label{lem:inter0}
		Let $f\colon \inte \cK \rightarrow \reals$ be a self-concordant function with constant $M_{f}>0$. Then, for any $\w, \p \in \inte \cK$ such that $r\coloneqq \|\p-\w\|_{\nabla^2 f(\w)}<1/M_{f}$, we have 
		\begin{align}
			\|\nabla f(\w) - \nabla f(\p)\|^2_{\nabla^{-2}f(\w)}  \leq \frac{1}{(1-M_{f} r )^{2}} \|\p- \w\|^2_{\nabla^{2}f(\w)}.
		\end{align}
	\end{lemma}
	The result of the lemma is reminiscent of the relationship between the Bregman divergence with respect to a function $f$ and the one with respect to its Fenchel dual $f^*$; that is, $D_{f}(\u, \v)=D_{f^*}(\nabla f(\v),\nabla f(\u))$ \cite[Proposition 11.1]{cesa2006prediction}.  
	\paragraph{Mixability.} As a by-product of our efficient solution to the portfolio problem, we present an algorithm that guarantees an adaptive logarithmic regret for mixable losses.
	\begin{definition}
		\label{de:mixable}
		A sequence $(f_t)\subset \{f\colon \cK \rightarrow \reals\}$ of convex functions is said to be $\eta$-mixable for $\eta>0$ if for any distribution $P$ on $\cK$, there exists $\u_*\in \cK$ such that 
		\begin{align}
			\forall t\geq 1, \quad 		f_t(\u_*) \leq - \eta^{-1} \log \E_{\u\sim P} e^{- \eta f_t(\u)}.
		\end{align}
	\end{definition}
	Formally, given an algorithm $\A$ whose outputs $(\u_t)$ achieve a logarithmic regret against any sequence of $\eta$-mixable losses, i.e.~$\sum_{t=1}^T (\ell_t(\u_t)-\ell_t(\u))\leq O(\ln T)$ for any $\u \in \cK$, we design a meta-algorithm that aggregates instances of $\A$ and generates outputs $(\w_t)$ that satisfy $\sum_{t\in I} (\ell_t(\w_t)-\ell_t(\u))\leq O(\ln^2 T)$, for any interval $I\subset [T]$ and $\u \in \cK$.
	\paragraph{Additional Notation.} For a differentiable convex function $f\colon \operatorname{int} \cK\rightarrow \reals$, we denote by $D_f(\u, \v)\coloneqq f(\u)-f(\v) - \inner{\nabla f(\v)}{\u - \v}$ the \emph{Bregman divergence} between $\u, \v \in \operatorname{int} \cK$ with respect to $f$. We use the notation $\wtilde{O}(\cdot)$ to hide poly-log factors in $T$ and $d$.
	
	\section{Background, Challenges, and Solution Sketch}
	\label{sec:sketchit}
	In this section, we start by describing the algorithm \adabar{} \cite{luo2018efficient} that we build on. We then point out key challenges we tackle to design our efficient portfolio selection algorithm. The analysis we sketch in \S\ref{sec:avoid} and \S\ref{sec:damped} is of independent interest as we discuss below. 
	
	\subsection{The \adabar{} Algorithm} 
	\label{sec:barrons}
	The \adabar{} algorithm consists of a base algorithm, $\mathsf{BARRONS}$, and a meta-algorithm that restarts the former under a certain condition on the sequence of returns and iterates of the algorithm.
	\paragraph{Base Algorithm.} $\mathsf{BARRONS}$ is simply mirror descent with a barrier regularizer. In particular, if we let $\bar \Delta_d \coloneqq \{\x\in \Delta_d\colon x_i \geq 1/T, \forall i \in[d]\}$, the outputs $(\p_t)$ of $\mathsf{BARRONS}$ are such that $
		\p_1 \coloneqq \mathbf{1}/d$ and $\p_{t+1} = \argmin_{\p \in \bar \Delta_d}    \inner{\p}{\g_t} + D_{\Phi_t}(\p,\p_t),$ where
		\begin{align} \Phi_t(\p)\coloneqq \sum_{i=1}^d \frac{- \ln p_i}{\eta_{t,i}} + \frac{d \|\p\|^2}{2}+\frac{\beta}{2}\sum_{s=1}^t \inner{\bnabla_t}{\p}^2  ,\ \   \ \ \eta_{t,i}\coloneqq \eta \cdot \max_{s\in[t]} e^{-\log _T(d p_{s,i})}, \label{eq:brun}
	\end{align}
	and $\bnabla_t \coloneqq \r_t/\inner{\r_t}{\p_t}$. Using the standard analysis of mirror descent and the fact that Cover's loss is exp-concave, \cite{luo2018efficient} show that the regret $R_T(\u)=\sum_{t=1}^T( \ell_t(\p_t)- \ell_t(\u))$ against a comparator $\u \in \bar \Delta_d$ (competing against comparators in $\bar \Delta_d$ is sufficient---see Lem.~\ref{lem:additive}) is bounded as 
	\begin{align}
		R_T(\u)\leq \sum_{t=1}^T \inner{\bnabla_t}{\p_t-\p_{t+1}} + \sum_{t=1}^T  (D_{\Phi_t}(\u, \p_t)- D_{\Phi_{t}}(\u, \p_{t+1}) -\beta \inner{\bnabla_t}{\p_t - \u}/2), \label{eq:decomp2}
	\end{align}
	as long as $\beta$, the parameter in the regularizer in \eqref{eq:brun}, is less than $\alpha_T(\u)\coloneqq \frac{1}{2}\wedge \min_{t\in[T]} \frac{1}{8 |\inner{\u - \p_t}{\bnabla_t}|}$. This condition seems strong since the algorithm does not have access to the sequence of returns $(\r_t)$ or the comparator $\u$ up-front to ensure that $\beta \leq \alpha_T(\u)$. However, this issue is resolved via a clever restart scheme as we describe further below. 
	
	The fact that the regularizers $(\Phi_t)$ have a quadratic term and a log-barrier ensures that the iterates $(\p_t)$ are stable. In particular, the first sum on the RHS of \eqref{eq:decomp2} can be bounded by $O(\beta^{-1} d \ln T)$. However, this term can still be problematic since $\beta^{-1}$ may be large; after all, \eqref{eq:decomp2} only holds when $\beta \leq \alpha_T(\u)$ and $\alpha_T(\u)$ may be as small at $1/(dT)$. 
	
	Fortunately, terms of the form $O(\beta^{-1} d \ln T)$ can be canceled by the second sum on the RHS of \eqref{eq:decomp2}, thanks to the log-barrier regularizer $\Psi_t(\p) \coloneqq \sum_{i=1}^d -\eta^{-1}_{t,i} \ln p_i$ in the definition of $\Phi_t$ and the non-decreasing nature of the learning rates $(\eta_{t,i})$. In particular, \cite{luo2018efficient} show that the second sum in \eqref{eq:decomp2} is bounded from above by $O\left(\eta^{-1}{d\ln T} \right)$ plus
	\begin{align}
		\sum_{t=1}^T (D_{\Psi_t}(\u,\p_t)- D_{\Psi_{t-1}}(\u, \p_t)){\leq}  - \frac{1}{8 \eta \ln T}\sum_{i\in[d]}\max_{t\in[T]} \frac{u_i}{p_{t,i}}. \label{eq:cancelling}
	\end{align}
where the inequality follows by \cite[proof of Lem.~6]{luo2018efficient}. Now the RHS of \eqref{eq:cancelling} can cancel the bound $O(\beta^{-1} d\ln T)$ on the stability term as long as $\beta \geq \alpha_T(\u)/2$ (see App.~\ref{sec:sketch} for details).   
	\paragraph{The Meta-Algorithm.} Since the sequence of returns $(\r_t)$ is not known up-front, it is not possible for any algorithm to pick $\beta$ so that the condition $\alpha_T(\u)/2\leq \beta \leq \alpha_T(\u)$ is always satisfied. Aggregating multiple instances of $\mathsf{BARRONS}$ with different $\beta$'s also fails since $\alpha_T(\u)$ depends on the outputs of the algorithm; and so, changing $\beta$ changes the target $\alpha_T(\u)$ for the base algorithm (see also discussion in \cite{luo2018efficient}). Instead of aggregating base algorithms, the approach taken by \cite{luo2018efficient} consists of restarting the base algorithm on round $t$ if the current estimate for $\beta$ satisfies $\beta > \alpha_t(\u_t)$, where $\u_t$ is the regularized leader:
	\begin{align}
		\u_t \in \argmin_{\u \in \bar \Delta_d}  \sum_{i=1}^d \frac{-\ln u_{i}}{\eta_{t,i}}  + \sum_{s=\tau}^t \ell_s(\u) \label{eq:FTRL}
	\end{align}
	and $\tau$ is the round where the current instance of the base algorithm was initialized. The technical reason for why this works is sketched in Appendix \ref{sec:sketch}.
	
	\paragraph{Computational Considerations.} 
	The computational complexity is dominated by the computation of the mirror descent iterates for the base algorithm and the FTRL computation \eqref{eq:FTRL} for the meta-algorithm. Both problems can be solved using an interior point method leading to a computational cost of $\wtilde O(d^{3.5} T + d^{2.5}T^2)$ after $T$ rounds. We will reduce the computational complexity to $\wtilde{O}(d^3 T)$ (where the $O(d^3)$ is due to the computation of a matrix inverse) by {\bf I)} avoiding the expensive FTRL computation in \eqref{eq:FTRL} thanks to a new adaptive algorithm for mixable losses; and {\bf II)} providing a new analysis for the damped Newton step to approximate mirror descent iterates. These techniques, which we describe next, are of independent interest. 
	
	\subsection{Avoiding the FTRL Computation}
	\label{sec:avoid}
	To avoid computing the regularized leader in \eqref{eq:FTRL} that is needed to trigger restarts, we will use a meta-algorithm that aggregates base algorithms initialized at different rounds (one may think of these as ``restarted'' instances of the base algorithm). If the meta-algorithm has a small regret against any of the base algorithms, then this would emulate the effect of performing restarts, without the expensive cost of FTRL computations. More formally, if we denote by $(\u_t^\tau)$ the outputs of an instance of the base algorithm $\A^\tau$ that is initialized at round $\tau$, we can emulate the effect of restarts if the outputs $(\u_t)$ of the meta-algorithm satisfy 
	\begin{align}
		\sum_{s=\tau}^t (\ell_s(\u_s)- \ell_s(\u_s^\tau)) \leq O(\operatorname{poly-log} (T)), \label{eq:targetted}
	\end{align}
	for all $\tau \in[T]$ and $t>\tau$. A regret bound of this type may be achieved using sleeping experts algorithms, where in our case the instance $\A^\tau$ is considered ``asleep'' during the rounds $s< \tau$ and ``awake'' for $s\geq \tau$. There are two challenges that come with using standard sleeping experts algorithms such as those in \cite{adamskiy2012closer,gaillard2014second}. First, such techniques operate on linearized losses, which is sufficient when seeking a $O(\sqrt{T})$ regret. This is not the case in our setting as we are aiming for a logarithmic regret. Second, if we want a regret bound of the form \eqref{eq:targetted} to hold for all $\tau\in[T]$, a naive sleeping experts strategy would require keeping track of $T$ experts. This would imply a $O(T)$-per-round computational complexity in the worst case, which would defeat the purpose of seeking an efficient alternative to the FTRL computation in \eqref{eq:FTRL}.
	
	We manage to circumvent these issues by presenting a new meta-algorithm that enjoys a logarithmic regret on any subinterval for mixable losses (this subsumes exp-concave losses). The algorithm is based on the recent work by \cite{zhang2019dual} that focuses on exp-concave online learning. To reduce the computational complexity, we show that it is sufficient to ensure a low regret against base algorithms indexed by a small set of geometric intervals \cite{daniely2015strongly}, reducing the number of experts at any round to at most $O(\ln T)$ (this is discussed in \S\ref{sec:strong}). 
	
	\subsection{Damped Newton Step for Mirror Descent}
	\label{sec:damped}
	Now that we have a way of avoiding the expensive FTRL computations of \adabar{}, it remains to find a more efficient alternative to the Mirror Descent (MD) computations of its base algorithm $\mathsf{BARRONS}$. Before describing how we use Damped Newton Steps (DNS) for this purpose, we first describe the shortcomings of existing approaches. 
	\paragraph{Shortcomings of previous DNS results.}
	\cite{abernethy2012interior} showed how one can use damped Newton steps to approximate the iterates $(\p_t)$ of FTRL given by $\p_{t+1}\in  \argmin_{\p\in \cC} f_{t+1}(\p) \coloneqq \sum_{s=1}^t \inner{\p}{\g_s}+ \Phi(\p)$, where $\Phi$ is a self-concordant barrier for some set of interest $\mathcal{C}$. In particular, \cite{abernethy2012interior} showed that for an appropriate scaling of $\Phi$ the damped Newton steps $(\w_t)$ defined by $\w_{t+1}=\w_t - \frac{1}{1+\lambda(\w_t, f_{t+1})} \nabla^{-2} f_{t+1}(\w_t) \nabla f_{t+1}(\w_t)$ are close enough to the FTRL iterates $(\p_t)$ so that the regret w.r.t.~$(\w_t)$ is bounded by the regret w.r.t.~$(\p_t)$ up to an additive $O(\sqrt{T})$. 
	
	We note that for a fixed regularizer $\Phi$ that is a self-concordant barrier for $\mathcal{C}$, it is known that the FTRL iterates in the previous paragraph match the MD ones; that is, $\p_{t+1}\in \argmin_{\p \in \mathcal{C}}  \inner{\p}{\g_t} +D_{\Phi}(\p, \p_{t})$, for all $t$. However, this is no longer the case when dealing with time-varying regularizers $(\Phi_t)$; that is, when $f_{t+1}(\cdot) =   \inner{\cdot}{\g_t} +D_{\Phi_t}(\cdot, \p_{t})$ (as in the case of $\mathsf{BARRONS}$). This means that we cannot directly use the analysis of \cite{abernethy2012interior} to show that damped Newton steps are good approximations of MD iterates with varying regularizers. What is more, the damped Newton steps with respect to $(f_{t})$ can no longer be computed directly in this case since the gradient $\nabla f_{t+1}(\w_t) =  \g_t + \nabla \Phi_t(\w_t)- \nabla \Phi_t(\p_t)$ in the expression of the DNS depends on the iterate $\p_t$, which is what we seek to efficiently approximate in the first place. Using $\w_t$ as an estimator for $\p_t$ does not work since the approximation errors accumulate across rounds in an unfavorable way (breaking the analysis of \cite{abernethy2012interior}). 
	
	One tempting approach around these issues is to target the FTRL iterates $(\p_t)$ given by $\p_{t+1}= \argmin_{\p\in \mathcal{C}} \sum_{s=1}^t \inner{\p}{\g_s} +\Phi_t(\p)$ instead of the MD ones. However, we are not aware of an existing analysis of FTRL that yields negative terms from Bregman divergences in the regret bound as in \eqref{eq:cancelling} (negative terms were needed to cancel the problematic $O(\beta^{-1}d \ln T)$ term in the regret bound).  
	
	\paragraph{Our approach.}
	Our solution consists of approximating FTRL iterates w.r.t.~modified gradients that are chosen in a way to still allow us to use the MD analysis to derive our regret bound (similar in spirit to the approach by \cite{foster2020adapting}). In particular, we consider the objective 
	\begin{align}
		f_{t+1}(\p) \coloneqq \sum_{s=1}^t \p^\top \left(\g_s - \nabla \Phi_s(\w_s) + \nabla \Phi_{s-1}(\w_s) \right) + \Phi_t(\p),
	\end{align} 
	with $(\w_s)$ being the damped Newton iterates w.r.t.~$(f_t)$. Under mild conditions on $(\Phi_t)$, the FTRL iterates $(\p_{t}\in \argmin_{\p \in \mathcal{C}}f_{t}(\p))$ can be efficiently approximated by $(\w_t)$. Despite the fact that the objective $f_{t+1}$ does not contain any ``unknown" MD iterates, we are still able to take advantage of the MD analysis and bring back negative terms from Bregman divergences (as in \eqref{eq:cancelling}) in the regret bound. The key fact that enables this is that the FTRL iterates $(\p_t)$ with respect to $(f_{t})$ match the MD iterates with modified gradients $(\tilde \g_t)$. In particular, $(\p_t)$ satisfy
	\begin{gather}
		\p_{t+1}\in \argmin_{\p} \inner{\tilde \g_t}{\p} + D_{\Phi_t}(\p, \p_{t}),\ \ \text{where} \ \ \tilde \g_t \coloneqq \g_t - \nabla \phi_t(\p_t)+\nabla \phi_t(\w_t), \label{eq:mirrormirror}
	\end{gather}
	and $\phi_t \coloneqq  \Phi_{t-1}-\Phi_{t}$. We now give a sketch of how this observation allows us to leverage the MD analysis and extract negative terms as in \eqref{eq:cancelling}. 
	\paragraph{MD analysis for FTRL (a sketch).} To illustrate how \eqref{eq:mirrormirror} helps in our analysis, consider the regularizer $\Phi_t$ in \eqref{eq:brun}, which we write as $\Phi_t = \Psi_t + \Theta_t$, where $\Psi_t$ is the barrier part $\Psi_t(\p)\coloneqq \sum_{i\in[d]}-\eta_{t,i}^{-1}\ln p_i$. We will further define $\psi_t\coloneqq \Psi_{t-1}-\Psi_t$ and note that the fact that $(\eta_{t,i})$ are non-decreasing, implies that $(\psi_t)$ are convex, self-concordant functions (the latter fact is all that is needed to generalize the current analysis). A key step in the analysis of the regret involves bounding the sum $\Sigma_T\coloneqq  \sum_{t=1}^T \inner{\g_t}{\w_t - \u}$ of linearized losses. For simplicity of the exposition, suppose that $\Theta_t\equiv 0$, for all $t$; the quadratic terms in $\Theta_t$ present no difficulty when it comes to bounding $\Sigma_T$ (see proof of Lem.~\ref{lem:decomp} for a derivation with non-zero $(\Theta_t)$). In this case, the sum $\Sigma_T$ can be written as 
	\begin{align}
	 \Sigma_T &= \sum_{t=1}^T \left(\inner{\g_t}{\w_t - \p_t}+ \inner{\tilde \g_t}{\p_t - \u}  +\inner{ \nabla \psi_t(\p_t)- \nabla \psi_t(\w_t)}{\p_t - \u} \right),\\
		 &= \sum_{t=1}^T \left(\inner{\g_t}{\w_t - \p_t} + \inner{\tilde \g_t}{\p_{t} - \p_{t+1}} \right)\nn \\ &\qquad  + \sum_{t=1}^T (\inner{\tilde \g_t}{\p_{t+1} - \u} + D_{\psi_t}(\u,\p_t)-  D_{\psi_t}(\u,\w_t) + D_{\psi_t}(\p_t,\w_t)),\label{eq:stok}
		\end{align}
	where we used the definitions of $\tilde \g_t$ and the Bregman divergence. Using this, and the facts that $\inner{\tilde \g_t}{\p_{t+1} - \u} \leq  D_{\Phi_t}(\u, \p_t)-D_{\Phi_t}(\u, \p_{t+1}) - D_{\Phi_t}(\p_{t+1},\p_t)$ (by optimality of $\p_{t+1}$---see proof of \cite[Lem.~5]{luo2018efficient}) and $D_{\psi_t}(\u, \p_t)= D_{\Phi_{t-1}}(\u, \p_t)- D_{\Phi_{t}}(\u, \p_t)$, we can bound $\Sigma_T$ as
		\begin{align}
	\Sigma_T	& \leq  \sum_{t=1}^T \left(  \inner{\g_t}{\w_t - \p_t} +  \inner{\tilde \g_t}{\p_{t} - \p_{t+1}}+ D_{\psi_t}(\p_t,\w_t)\right)\nn \\ \displaybreak[0]
		& \ \  +D_{\Phi_0}(\u,\p_1) - D_{\Phi_T}(\u,\p_{T+1})+\sum_{t=1}^T ( D_{\Phi_{t}}(\u,\w_t)-  D_{\Phi_{t-1}}(\u,\w_t)). \label{eq:fluffy}
	\end{align} 
Thus, one can extract negative terms (as in \eqref{eq:cancelling}) from the last sum in \eqref{eq:fluffy} to cancel the $O(\beta^{-1}d\ln T)$ term in the regret bound. The remaining terms in \eqref{eq:fluffy} can be shown to be small thanks to I) $(\w_t)$ approximate $(\p_t)$ well, II) stability of the mirror descent iterates, and III) the fact that $\psi_t$ is self-concordant together with Lem.~\ref{lem:inter0}. 
Next, we present the full details and guarantee of our algorithm.

	\section{An Efficient Algorithm for Online Portfolio Selection}
	\label{sec:fulldetails}
	Our final algorithm (Alg.~\ref{alg:DONSmeta}) will consist of a set of base algorithms (instances of Alg.~\ref{alg:DONSbase}) and a meta algorithm (instance of Alg.~\ref{alg:DONSmetageneral}). We analyze these algorithms separately in the next two subsections before combining the results in Subsection \ref{sec:Final}.

	\subsection{Base Algorithm: Damped Online Newton Step ($\mathsf{DONS}$)}		
	To analyze our base algorithm (Alg.~\ref{alg:DONSbase}), we will consider the following sequence of regularizers that are defined in terms of the iterates $(\u_t)$ and $(\w_t)$ in Algorithm \ref{alg:DONSbase}; the sequence of observed return vectors $(\r_t)$; and the gradients $(\g_t)\equiv (J \r_t/\inner{\r_t}{\bar \u_t})$:
	\begin{align}
		\Phi_t(\u) \coloneqq  \Psi_{t}(\u)  + \frac{\beta d\|\u\|^2}{8}    + \frac{\beta}{8}\sum_{s=1}^t\inner{\g_t}{\u-\w_t}^2, \quad \forall \u \in \cC_{d-1}.\label{eq:regularizer}   \\
		\text{where}\ \  \Psi_t(\x) \coloneqq -\sum_{i=1}^d \frac{-\ln \bar x_{i}}{\eta_{t,i}},\quad   \quad  \eta_{t,i} \coloneqq \eta  \cdot e^{\log_{T} (\rho_{t,i}/d)},
		\label{eq:defs}
	\end{align}
	and $\rho_{t,i}$ is such that $\rho_{t,i} \in \left[ \max_{s\in[t]} (2\bar u_{s,i})^{-1}, \max_{s\in[t]} (\bar u_{s,i})^{-1}\right]$, for all $i\in[d]$, and $\bm{\rho}_0\coloneqq d \mathbf{1}$. In particular, for every $i\in[d]$, $(\rho_{t,i})_t$ satisfies the recursion 
	\begin{align}
		\rho_{t,i}=\mathbb{I}\{2 \rho_{t-1,i}  < \tfrac{1}{\bar u_{t,i}}  \} \cdot \tfrac{1}{\bar u_{t,i}} +\mathbb{I}\{2 \rho_{t-1,i} \geq \tfrac{1}{\bar u_{t,i}}\} \cdot \rho_{t-1,i}. \label{eq:choice}
	\end{align} 
	\cite{luo2018efficient} chose the sequence $(\rho_{t,i})$ such that $\rho_{t,i}=\max_{s\in[t]}1/\bar u_{s,i}$, $\forall i\in[d]$ and $t\in[T]$. Using our new analysis of the damped Newton steps for MD, this choice leads to a regret bound of order $\wtilde O(d^m)$ with $m>2$, which is worse than what we are aiming for. Our choice in \eqref{eq:choice} ensures that the barrier $\Psi_t$ changes at most $O(d \ln T)$ times, which is crucial to proving the desired bound.

	For any $t\in[T]$, the output $\u_{t+1}$ of Algorithm \ref{alg:DONSbase} at round $t+1$ is given by $\u_{t+1}=\w_{t+1}'=(1-1/T)\w_{t+1}+\mathbf{1}/(dT)$, where $\w_{t+1}$ is the damped Newton step:
	\begin{gather}
		\w_{t+1} = \w_t  - \frac{\nabla^{-2}\Phi_t(\w_t) \bm{\nabla}_{t}}{1+ 4\sqrt{e\eta}\|\bm{\nabla}_{t}\|_{\nabla^{-2}\Phi_t( \w_t)}} ,  \label{eq:damped} 
	\end{gather}
	with $\bm{\nabla}_t \coloneqq  \nabla \Phi_t(\w_t) + \sum_{s=1}^t  \left(\g_s - \nabla \Psi_{s}(\w_s)+\nabla \Psi_{s-1}(\w_s) \right)$ and $\w_1 = \mathbf{1}/d$. In part due to the fact that $\operatorname{dom} \Phi_t=\cC_{d-1}$, for all $t\geq 1$, the iterates $(\w_t)$ in Algorithm \ref{alg:DONSbase} are only well defined when the update rule in \eqref{eq:damped} ensures that $\w_{t+1}\in \cC_{d-1}$ for any $\w_t\in \cC_{d-1}$. This is in fact the case as we show next by leveraging the self-concordant property of $\Phi_t$. To simplify notation in the proof of the next lemma (which is in App.~\ref{sec:cproofs}) and in the rest of the paper, we let $\vartheta_i\colon \cC_{d-1}\rightarrow \reals$ be defined by \begin{align}\vartheta_i(\x)\coloneqq - \ln \bar x_i, \quad \forall i \in[d].
		\label{eq:thepsi}
	\end{align}
Note that the self-concordant barrier in \ref{eq:defs} satisfies $\Psi_t (\cdot)= \sum_{i\in[d]} \vartheta_i(\cdot)/\eta_{t,i}$.
	\begin{lemma}
		\label{lem:selfconcord}
		For all $t\geq 1$, $\Phi_t$ in \eqref{eq:regularizer} is a self-concordant function with constant $M_{\Phi_t} \leq \sqrt{\eta e}$. 
	\end{lemma}

	\begin{algorithm}
		\caption{$\mathsf{DONS}$ (Base Algorithm): Damped Online Newton Step for Portfolio Selection.}
		\label{alg:DONSbase}
		\begin{algorithmic}[1]
			\REQUIRE Parameters $\eta, \beta>0$. 
			\STATE Set $\w_1= \mathbf{1}/d\in \reals^{d-1}$, $\bm{\rho}_0=d \mathbf{1}\in \reals^d$, $\G_0=\bm{0}$, and $V_0 = \beta  d I/4 \in \reals^{d-1\times d-1}$.
			\FOR{$t=1,2,\dots$}
			\STATE Play $\u_t =(1-\frac{1}{T})\w_t+\frac{1}{dT}\mathbf{1}$ and observe gradient $\g_t = \nabla \ell_t(\u_t)= J\r_t/\inner{\r_t}{\bar \u_t}$. \label{line:mix} 
			\STATE Set $\rho_{t,i} = \mathbb{I}\{2\rho_{t-1,i}<\frac{1}{\bar u_{t,i}} \}\cdot \frac{1}{\bar u_{t,i}} + \mathbb{I}\{2 \rho_{t-1,i}\geq \frac{1}{\bar u_{t,i}} \}\cdot  \rho_{t-1,i}$, for all $i\in[d]$.  \label{line:therho}
			\STATE Define $\Psi_t(\x)=  -\sum_{i=1}^d \frac{\ln \bar x_{i}}{\eta_{t,i}}$, where $\eta_{t,i} \coloneqq \eta  \cdot \exp({\log_{T} (\rho_{t,i}/d)})$, $\forall i\in[d]$. 
			\STATE Set $\G_{t}=\G_{t-1}+\g_t\cdot  (1-\beta\inner{\g_t}{\w_t}/4) - \nabla \Psi_{t}(\w_t)+\nabla \Psi_{t-1}(\w_t) $.
			\STATE Set $V_t = V_{t-1} + \beta \g_t \g_t^\top/4$ and  $\bm{\nabla}_t = \G_t + V_t \w_t +\nabla \Psi_t(\w_t) + \beta d \w_t/4$.
			\STATE \label{line:interiter}Set $ \w_{t+1}= \w_t - \frac{1}{1+ 4\sqrt{\eta e}\|\bm{\nabla}_t \|_{(\nabla^2\Psi_t( \w_t)+V_t)^{-1}}} {(\nabla^2\Psi_t( \w_t)+V_t)^{-1} \bm{\nabla}_t }$. 
			\ENDFOR
		\end{algorithmic}
	\end{algorithm}
	
	\paragraph{Damped Online Newton Steps as Approximate MD Iterates.}
	A key part of our analysis consists of showing that the intermediate iterates $(\w_t)$ on Line~\ref{line:interiter} of Algorithm \ref{alg:DONSbase} are close to the mirror descent iterates $(\p_t)$ with respect to the sequence of regularizers $(\Phi_t)$ in \eqref{eq:regularizer}:
	\begin{gather}
		\p_{t+1} \in \argmin_{\p \in \cC_{d-1}} F_{t+1}(\p) \coloneqq  \inner{\p}{\tilde \g_t}+ D_{\Phi_t}(  \p,   \p_t), \quad \text{where}\label{eq:mirrordescent} \\
		\tilde \g_t \coloneqq  (1+\beta \inner{\g_t}{\p_t-\w_t}/4)\g_t + \sum_{i=1}^d \left(\frac{1}{\eta_{t,i}} - \frac{1}{\eta_{t-1,i}} \right) (\nabla \vartheta_i(\p_t)- \nabla \vartheta_i(\w_t)), \label{eq:tildedgt}
	\end{gather}
	and $\p_1\coloneqq\mathbf{1}/d$ (recall $(\vartheta_i)$ from \eqref{eq:thepsi}). Next, we formally state this result (recall $\lambda(\cdot,\cdot)$ from \S\ref{sec:prelim}):

	\begin{lemma}
		\label{lem:close}
		For any $\beta \in(0,1/8)$ and $\eta \leq 1/2^{14}$, the iterates $(\w_t)$ in Algorithm \ref{alg:DONSbase} satisfy, 
		\begin{gather}
			\forall t \geq 1, \quad 	\frac{\|\w_t-\p_t  \|_{\nabla^2 F_t(\w_t)} }{2^4\sqrt{e \eta}}\leq \frac{	\lambda(\w_t, F_{t})}{2^3 \sqrt{e\eta}} \leq \lambda(\w_{t-1}, F_{t})^2 \leq C \eta,\label{eq:workingftrl} 
		\end{gather}
		where $C\coloneqq \frac{4e}{4^{-1}\vee (1-1/T)^2}$. Further, we have  $\sum_{t=1}^T \|\w_t - \p_t\|_{\nabla^2 \Phi_{t-1}(\w_t)}^2  \leq 1+ {15 \beta^{-1} d\log T}$.
	\end{lemma}
	The next lemma, which will be useful in the proof of Lemma \ref{lem:close}, essentially shows that the mirror descent iterates in \eqref{eq:mirrordescent} match the FTRL iterates with respect to $(f_t)$, where $f_{t+1}(\p) \coloneqq \sum_{s=1}^t \p^\top (\g_s - \nabla \Phi_s(\w_s) + \nabla \Phi_{s-1}(\w_s))  + \Phi_t(\p)$ (c.f.~discussion in \S\ref{sec:damped}).
	\begin{lemma}
		\label{lem:wom}
		For all $t\in [T]$, we have $\nabla F_{t+1}(\w_t) = \g_t + \nabla F_t(\w_t)$ and  for all $\w\in \cC_{d-1}$,
		\begin{align}
			\nabla F_{t+1}(\w)  = \nabla \Phi_t(\w)+ \sum_{s=1}^t( \g_s  - \nabla \Psi_s(\w_s)+ \nabla \Psi_{s-1}(\w_s)).
		\end{align}
	\end{lemma}

	\paragraph{Regret Decomposition.}  We now present the main regret decomposition. In the proof, which is Appendix \ref{sec:proofs0}, we follow similar steps as the ones outlined in \S\ref{sec:damped}.
	\begin{lemma}
		\label{lem:decomp}
		Let $T>1$, $c_T\coloneqq 1-1/T$, and $\psi_t \coloneqq \Psi_{t-1} -\Psi_t$. Further let $(\u_t)$ and $(\w_t)$ be as in Algorithm~\ref{alg:DONSbase} with parameters $\eta, \beta>0$, and $(\p_t)$ as in \eqref{eq:mirrordescent}. For any sequence of returns $(\r_t)$ and $\u \in \cC_{d-1}$ such that $\beta \leq 8^{-1}\wedge  |8\inner{\g_t}{\u_t-\u''}|^{-1}$ (recall $\u',\u''$ from \S\ref{sec:prelim}), we have, for all $t\in[T]$,
		\begin{align}
			\sum_{t=1}^T	\frac{\ell_t(\u_t)  - \ell_t(\u)}{c_T} &  \leq \sum_{t=1}^T \inner{\tilde \g_t}{\p_t - \p_{t+1}}  + \sum_{t=1}^T (D_{\psi_t}(\p_t, \w_t)+D_{\Psi_t}(\u',\w_t)- D_{\Psi_{t-1}}(\u',\w_t))  \nn \\ & \quad +O\left(\frac{d \ln T}{\eta} \right) + \sum_{t=1}^T \inner{\g_t}{\w_t - \p_t} + \frac{3\beta}{8} \inner{\g_t}{\w_t- \p_t}^2. \label{eq:trip}
		\end{align}
	\end{lemma}
	Next, we bound each term in this decomposition starting with $\sum_{t=1}^T D_{\psi_t}(\p_t,\w_t)$. To show that this term is small, we rely on the fact that $\psi_t$ is a self-concordant functions, which holds true by our choice of ``doubling'' $(\rho_{t,i})$'s in \eqref{eq:choice}. This enables us to relate $D_{\psi_t}(\p_t,\w_t)$ to $\|\w_t-\p_t\|_{\nabla^2 \Phi_{t-1}(\w_{t})}$ via Lem.~\ref{lem:inter0}, which we can then bound using Lem.~\ref{lem:close}.
	\begin{lemma}
		\label{lem:concord}
		Let $(\Psi_t)$ be as in \eqref{eq:defs}. Then, the function $\psi_t\coloneqq \Psi_{t-1}-\Psi_t$ is a self-concordant function with constant $\sqrt{e \eta\log_2 T} $. Furthermore, the mirror descent iterates $(\p_t)$ in \eqref{eq:mirrordescent} and the iterates $(\w_t)$ in Algorithm \ref{alg:DONSbase} satisfy $\sum_{t=1}^T D_{\psi_t}(\p_t, \w_t)\leq O(d \ln T)$. 
	\end{lemma}
	We move to the stability term $\sum_{t=1}^T	\inner{\tilde \g_t}{ \p_t - \p_{t+1}}$, which is the most technical one to bound due to the modified gradients $(\tilde \g_t)$. We will use use H\"older's inequality and the triangle inequality to bound $\inner{\tilde \g_t}{ \p_t - \p_{t+1}}$ in terms of $\|\p_t-\p_{t+1}\|_{\nabla^2 \Phi_t(\w_t)}$, $\|\g_t\|_{\nabla^{-2}\Phi_t(\w_t)}$, and $\|\nabla \vartheta_i(\p_t) -\nabla \vartheta_i(\w_t)\|_{\nabla^{-2}\Phi_t(\w_t)}$. Then, we use the self-concordance property in Lem.~\ref{lem:inter0} to relate the latter term to $\|\w_t-\p_t\|_{\nabla^2 \Phi_t(\w_t)}$, which (thanks to Lem.~\ref{lem:close}) will allow us to show that the stability term is small.
	\begin{lemma}
		\label{lem:stabi}
		Let $T>1$ and $(\tilde \g_t)$ be as in \eqref{eq:tildedgt}. If $\eta \leq 1/2^{14}$ and $\beta \in(0,1/8) $, then the iterates $(\p_t)$ in \eqref{eq:mirrordescent} satisfy $
		\sum_{t=1}^T 	\inner{\tilde \g_t}{ \p_t - \p_{t+1}}  \leq \frac{18 d \ln T}{\beta} + O(d \ln T).$
	\end{lemma}
	We now bound the sum divergences which will allow us to cancel the undesirable $O(\beta^{-1} d\ln T)$ term in the regret bound as discussed in \S\ref{sec:barrons}.
	\begin{lemma}
		\label{lem:decomplemnew}
		Let $T>1$ and $(\u_t)$ be the iterates of Alg.~\ref{alg:DONSbase} with parameters $\beta\in(0,1/8)$ and $\eta\leq 1/2^{14}$. For any sequence $(\r_t)$, the iterates $(\p_t)$ in \eqref{eq:mirrordescent} satisfy (recall $\u'$ and $\u''$ from \S\ref{sec:prelim}) $	\sum_{t=1}^T (D_{\Psi_t}(\u', \w_t) - D_{\Psi_{t-1}}(\u', \w_{t})) \leq   \frac{-1}{16 \eta \ln T} \sum_{i=1}^d \max_{t\leq T}  \frac{\bar u''_{i}}{\bar u_{t,i}}   + O\left({d}/{\eta} \right)$, for all $ \u\in \cC_{d-1}$.
	\end{lemma}
	It remains to upper bound the sums $\sum_{t=1}^T\inner{\g_t}{\w_t-\p_t}^i$, for $i\in\{1,2\}$ which are expected to be small since $(\w_t)$ are close to $(\p_t)$ by Lemma \ref{lem:close}:
	\begin{lemma}
		\label{lem:pseudostab}
		Let $T>1$, $c_T\coloneqq 1-1/T$, and $S_T\coloneqq 1+ 4{\sqrt{15}\beta^{-1} d \ln T}$. Further, let $(\w_t)$ be the iterates in Alg.~\ref{alg:DONSbase} with parameters $\beta\in(0,1/8)$ and $\eta\leq 1/2^{14}$. For any sequence of returns $(\r_t)$, the mirror descent iterates in \eqref{eq:mirrordescent} satisfy $
			\sum_{t=1}^T\inner{\g_t}{\w_t-\p_t}\leq  S_T$ and $ 	\sum_{t=1}^T\inner{\g_t}{\w_t-\p_t}^2 \leq  \frac{64 e^2 \eta^{2}}{ c^3_T}S_T$.
	\end{lemma}
	Combining these results, we obtain the following regret bound for our base algorithm:
	\begin{theorem}[Base Algorithm Regret]
		\label{lem:thekey0}
		Let $T>1$, and $(\u_t)$ be the iterates of Algorithm \ref{alg:DONSbase} with parameters $\beta\in(0,1/8)$ and $\eta\leq 1/2^{14}$. For any sequence of returns $(\r_t)$ and $\u \in \cC_{d-1}$ such that, for all $t\in[T]$, $\beta \leq 8^{-1} \wedge |8\inner{\g_t}{\u_t-\u''} |^{-1}$ (recall $\u'$ and $\u''$ from \S\ref{sec:prelim}), we have
		\begin{align}
			\sum_{t=1}^T (\ell_t(\u_t) - \ell_t(\u)) \leq  	O\left(\frac{d \ln T}{\eta} \right)   + \frac{34  d \ln T}{\beta} - \frac{1}{32 \eta \ln T} \sum_{i=1}^d \max_{t\leq T} \frac{\bar u''_i}{\bar u_{t,i}}.
		\end{align}
	\end{theorem}
The regret bound in Theorem \ref{lem:thekey0} is the same as that of $\mathsf{BARRONS}$ up to constant factors. We are now going to describe our adaptive meta-algorithm which allows us to emulate the effect of restarts in \adabar{}, as discussed in \S\ref{sec:avoid}.
	
	\subsection{Adaptive Meta-Algorithm for Mixable Losses ($\mathsf{AdaMix}$)}
	\label{sec:strong}
	Let $\mathcal{I}$ be the geometric covering intervals $
	\cI\coloneqq \bigcup_{i,k\in \mathbb{N}} \left\{[2^k i,2^k(i+1)-1]\right\}$ suggested by \cite{daniely2015strongly}.
	We further define the ``restriction'' of $\cI$ to $[T]$ as $\cI|_T \coloneqq \{I\cap [T]\ :\ I\in \cI\}\cup \{[T]\}$.  
	To build our final meta algorithm, we introduce a new algorithm (Alg.~\ref{alg:DONSmetageneral}) that can achieve an adaptive, logarithmic regret in the expert setting with mixable losses (see Def.~\ref{de:mixable}). That is, we present an algorithm that enjoys a logarithmic regret on any interval $I \subseteq [T]$ when the losses are mixable. Alg.~\ref{alg:DONSmetageneral} takes in a set of base algorithms/experts $(\A^{\beta,I})$, where for expert $\A^{\beta, I}$, $\beta$ represents a parameter in some predefined grid $\cG\subset \reals$ and $I\in \cI|_T$ represents the interval on which the expert is active; in this case, expert $\A^{\beta,I}$ is initialized at round $t=\min I$ and terminates after round $t=\max I$. We will state the guarantee of Alg.~\ref{alg:DONSmetageneral} when the substitution function $\Upsilon\colon \bigcup_{\cJ \subseteq \cI|_T} \Delta(\cG \times \cJ)\times \cU^{|\cG|\times |\cJ|} \rightarrow \cU$, for the losses $f_t\colon \cU\rightarrow \reals$, satisfies, for all $\cJ\subseteq \cI|_T$, $Q\in \triangle(\cG\times \cJ)$, and $U\coloneqq (\u^{\beta,I})\in \cU^{|\cG|\times |\cJ|}$, 
	\begin{align}
		f_t(\Upsilon(Q,U)) \leq -\eta^{-1} \log \sum_{\beta\in \cG, I \in \cJ } Q^{\beta, I} e^{-\eta f_t(\u^{\beta,I})}, \forall t\geq 1.\label{eq:sub}
	\end{align}
	\begin{algorithm}[t]
		\caption{$\mathsf{AdaMix}$: Adaptive Meta-Algorithm for Mixable Losses.}
		\label{alg:DONSmetageneral}
		\begin{algorithmic}
			\REQUIRE \textbf{I)} Grid $\cG$ of $\beta$ values, horizon $T$, and $\eta>0$; \textbf{II)} Instances $(\A^{\beta,I})_{\beta \in \cG, I \in \cI|_T}$, where $\A^{\beta,I}$ is active during $I$; \textbf{III)} a substituting function $\Upsilon\colon \bigcup_{\cJ \subseteq \cI|_T} \Delta(\cG \times \cJ)\times \cU^{|\cG|\times |\cJ|} \rightarrow \cU$ for the losses $(f_t\colon \cU\rightarrow \reals)$.\algcomment{The substitution function $\Upsilon$ will be chosen to satisfy \eqref{eq:sub}.}
			\STATE Set $\cJ_{0}=\emptyset$.
			\FOR{$t=1,\dots, T$}
			\STATE Identify the set of newly active intervals $\tilde \cJ_t=\{ I \in \cI|_{T}\colon  \min I =t \}$.
			\STATE Set $\cF_{t-1}^{\beta,I}=0$ for all $\beta \in \cG$ and $I \in \tilde \cJ_t$.
			\STATE Update the set of active intervals $\mathcal{J}_t=\mathcal{J}_{t-1}\cup \tilde \cJ_t$.
			\STATE Receive $\u_t^{\beta, I}\in \cU$ from each $\mathsf{A}^{\beta, I}$ such that $\beta \in \cG$ and $I \in \cJ_t$.
			\STATE Set $q^{\beta,I}_{t-1}=e^{-\eta \cF^{\beta,I}_{t-1}}/Z_t$ for $\beta \in \cG$ and $I \in \cJ_t$, where $Z_t= \sum_{\beta\in \cG, I\in \cJ_t}   e^{- \eta \cF^{\beta, I}_{t-1}}$.
			\STATE Play $\u_{t} = \Upsilon\left(Q_{t-1},U_t\right)$, where $Q_{t-1}\coloneqq (q_{t-1}^{\beta,I})_{\beta\in \cG,I\in \cJ_t}$ and $U_t\coloneqq (\u_{t}^{\beta,I})_{\beta\in \cG,I\in \cJ_t}$.
			\STATE Observe $f_t$ and set $\cF_{t}^{\beta, I} =\cF_{t-1}^{\beta, I} + f_t(\u_t^{\beta, I}) - f_t(\u_t)$, for all $\beta \in \cG$ and $I \in \cJ_t$.
			\STATE Send $f_t$ to each $\A^{\beta, I}$ such that $\beta \in \cG$ and $I \in \cJ_t$.
			\STATE  Update the set of active intervals $\mathcal{J}_t=\mathcal{J}_{t}\setminus \{   I \in \cI|_T\colon  \max I =t \}$. 
			\ENDFOR
		\end{algorithmic}
	\end{algorithm}
Such a substitution function is guaranteed to exist when the losses $(f_t)$ are $\eta$-mixable (see Def.~\ref{de:mixable}). For the case of Cover's loss \eqref{eq:cover}, which is 1-mixable (in fact 1-exp-concave), we will set the substitution function to $\Upsilon(Q,U)=\sum_{\beta,I}Q^{\beta,I}U^{\beta,I}$, which satisfies \eqref{eq:sub} with $\cU=\cC_{d-1}$ and $(f_t)\equiv (\ell_t)$. 
	\begin{proposition}
		\label{prop:mainprop}
		Let $\eta>0$ and $\cG$ be a set s.t.~$|\cG|\leq M$. Further, let $(\u^{\beta,I}_t)_{t\in I}$, $\beta \in \cG$ and $I\in \cI$, be the outputs of the subroutine $\mathsf{A}^{\beta,I}$ within Alg.~\ref{alg:DONSmetageneral} in response to a sequence of $\eta$-mixable losses $(f_t)$. Then, the outputs $(\u_t)$ of Algorithm \ref{alg:DONSmetageneral} with a substitution function $\Upsilon$ satisfying \eqref{eq:sub}, guarantee
		\begin{align}
			\sum_{s\in I \cap [t]} (f_s( \u_s)- f_s( \u_s^{\beta, I})) \leq  (2 \ln t+\ln M)/\eta,\quad \text{for all } I\in \mathcal{I}, \beta \in \cG, \text{ and } t\in I.  \label{eq:firstfirst}
		\end{align}
	\end{proposition}
	Now, an adaptive, logarithmic regret for mixable losses follows easily from this proposition. To see this, let $J$ be any interval in $[T]$. Then, by \cite{daniely2015strongly} we know that there exist disjoint sets $I_1,\dots, I_N \in \cI$ such that $\bigcup_{i\in[N]}I_i=J$ and $|N|\leq O(\ln T)$. Now, if any subroutines $\mathsf{A}^{\beta,I}$ achieves a logarithmic regret $\cR^{\beta}_I(\u)\leq O(\eta^{-1}{\ln T})$ within the interval $I$ against any comparator $\u$, \eqref{eq:firstfirst} implies that $\sum_{s\in J} (f_s(\u_s)- f_s(\u)) \leq \frac{1}{\eta}(2|N|\ln T+|N| \ln M)+  \sum_{i\in[N]} \cR^{\beta}_{I_i}(\u) \leq O( \frac{\ln^2 T}{\eta})$.
	\subsection{Finally Algorithm and Guarantee}
	\label{sec:Final}
	Next, we will instantiate Alg.~\ref{alg:DONSmetageneral} with Cover's loss and the base algorithms $(\A^{\beta, I})$ set as instances of Alg.~\ref{alg:DONSbase} with $\eta=1/(286^2 d\ln^3 T)$ and $\beta \in \cG$, where $
	\cG \coloneqq \{\frac{1}{d 2^{i+3}}\colon i\in [\ceil{\log_2 T}]\}.$
	\begin{algorithm}
		\caption{$\mathsf{AdaMix}+\mathsf{DONS}$: Adaptive Meta-Algorithm for Online Portfolio Selection.}
		\label{alg:DONSmeta}
		\begin{algorithmic}
			\STATE Instantiate Alg.~\ref{alg:DONSmetageneral} with \textbf{I)} $(f_t)\equiv (\ell_t)$ ($\ell_t$ as in \eqref{eq:cover}) with $\cU=\cC_{d-1}$; \textbf{II)} $(\A^{\beta,I})$ set as instances of Alg.~\ref{alg:DONSbase} with $\eta = \frac{1}{286^2 d \ln^3 T}$, $\beta \in \cG$, and $I\in \cI|_{T}$; and \textbf{III)} $\Upsilon(Q,U)=\sum_{\beta,I}Q^{\beta,I}U^{\beta,I}$.
		\end{algorithmic}
	\end{algorithm}
	Furthermore, we will set the substitution function to $\Upsilon(Q,U)=\sum_{\beta,I}Q^{\beta,I}U^{\beta,I}$, which satisfies \eqref{eq:sub} with $\cU=\cC_{d-1}$ and $(f_t)\equiv (\ell_t)$. In particular, the outputs $(\u_t)$ of Alg.~\ref{alg:DONSmetageneral} in this setting can be written as
	\begin{align}
		\u_{t} \coloneqq \frac{\sum_{\beta, I \colon  t \in I } \exp(- \cL^{\beta, I}_{t-1}  ) \u^{\beta,I}_t }{\sum_{\beta, I \colon  t\in I}  \exp(- \cL^{\beta, I}_{t-1})}, \quad
		\text{where} \quad  \cL^{\beta, I}_t \coloneqq \sum_{s\in I \cap [t]} (\ell_s( \u^{\beta, I}_s) - \ell_s( \u_s)).	\label{eq:mixture}
	\end{align}
Since the regret of $\mathsf{DONS}$ is the same as that of $\mathsf{BARRONS}$ (see discussion after Thm.~\ref{lem:thekey0}), and the adaptive regret enabled by Alg.~\ref{alg:DONSmetageneral} allows us to emulate the restarts of \adabar{} (see \S\ref{sec:avoid}), the regret of our final Alg.~\ref{alg:DONSmeta} will be the same as that of \adabar{} up to log factors.

	\begin{theorem}
		\label{thm:mainthm}
		The regret of Algorithm \ref{alg:DONSmeta} is bounded by $O(d^2 \ln^5 T)$. Furthermore, the algorithm runs is $O(d^3 \ln^2 T)$ per round and requires $O(d \ln^2 T)$ total space.
	\end{theorem}
	\clearpage
	\section*{Acknowledgement}
We acknowledge support from the ONR through awards N00014-20-1-2336 and N00014-20-1-2394.
	
	\DeclareRobustCommand{\VAN}[3]{#3} 
	\bibliography{biblio}
	
	\DeclareRobustCommand{\VAN}[3]{#2} 
	
	\clearpage
	\appendix
	
	\tableofcontents
	\clearpage

	\section{Technical Lemmas}
	\begin{lemma}
		\label{lem:iterates}
		Let $\Psi(\x)\coloneqq -\sum_{i=1}^d \ln \bar x_i$ and $f(\w) \coloneqq -\log \inner{\r}{\bar \w}$, for $\r \in [0,1]^d$, and $\x,\w \in \cC_{d-1}$. Then, for all $\mu \in (0,1)$ and $\tilde \w\in \cC_{d-1}$, we have
		\begin{align}
			\nabla f(\w) \nabla f(\w)^\top \preceq (1-\mu)^{-2}\nabla^2 \Psi(\tilde\w), \quad \text{where} \quad  \w \coloneqq (1-\mu)\tilde \w+ \mu \mathbf{1}/d. \label{eq:firstclaim}
		\end{align}
		Furthermore, $\|\nabla f(\w)\|^{2}_{\nabla^{-2} \Psi(\tilde \w)} \leq (1-\mu)^{-2}$.
	\end{lemma}
	\begin{proof}[{\bf Proof}]
		Since $\|\nabla f(\w)\|^{2}_{\nabla^{-2}\Psi(\tilde \w)}  =\nabla f(\w)^\top \nabla^{-2}\Psi(\tilde\w) \nabla f(\w)$, the second claim follows from \eqref{eq:firstclaim}. Thus, it suffices to show that $\nabla f(\w) \nabla f(\w)^\top \preceq (1-\mu)^{-2}\nabla^2 \Psi(\tilde\w)$. Letting $\tilde \r \coloneqq (r_1,\dots, r_{d-1})$, we have 
		\begin{align}
			\nabla f(\w) \nabla f(\w)^\top  &= (\tilde \r - r_{d}\mathbf{1}) (\tilde \r - r_{d}\mathbf{1})^\top/\inner{\r }{\bar \w }^2,\nn \\
			& = \left( \tilde \r  \tilde \r ^\top + r_{d}^2 \mathbf{1}\mathbf{1}^\top - r_{d} \tilde \r   \mathbf{1}^\top  -r_{d} \mathbf{1} \tilde \r ^\top \right)/\inner{\r }{\bar \w }^2 ,\nn \\
			&  \preceq \left( \tilde \r  \tilde \r ^\top + r_{d}^2 \mathbf{1}\mathbf{1}^\top \right)/\inner{\r }{\bar \w }^2 ,\nn \\
			&   \preceq \left( \tilde \r  \tilde \r ^\top + r_{d}^2 \mathbf{1}\mathbf{1}^\top \right)/\inner{\r }{\bar \w }^2, \nn \\
			& \preceq  \tilde \r  \tilde \r ^\top/\inner{\tilde \r }{ \w }^2  +  \mathbf{1}\mathbf{1}^\top /w_{d}^2,\\
			& \preceq  \tilde \r  \tilde \r ^\top/\inner{\tilde \r }{ \w }^2  + (1-\mu)^{-2} \mathbf{1}\mathbf{1}^\top /\tilde w_{d}^2. \label{eq:lastt}
		\end{align} 
		We will now show that $\tilde \r  \tilde \r ^\top/\inner{\tilde \r }{\w }^2 \preceq \text{diag}(1/w_{1}^2, \dots, 1/w_{d-1}^2)$. For this, it suffices to show that for any vector $\u \in \reals^d$, we have
		$(\sum_{i=1}^{d-1}  r_{i} u_{i})^2/  (\sum_{i=1}^{d-1}  r_{i} w_{i})^2 \leq \sum_{i=1}^{d-1} u^2_i/w_{i}^2$. This is indeed the case; by Cauchy Schwarz, we have 
		\begin{align}
			\left(\sum_{i=1}^{d-1}  r_{i} u_{i} \right)^2 &\leq  \left(\sum_{i=1}^{d-1}  r^2_{i} w^2_{i}\right)\left( \sum_{i=1}^{d-1}u^2_{i}/ w^2_{i}\right), \nn \\
			& \leq \left( \sum_{i=1}^{d-1}  r_{i} w_{i}\right)^2 \left(\sum_{i=1}^{d-1}u^2_{i}/ w^2_{i}\right),\nn
		\end{align}
		where the last inequality follows by the fact that $r_{i},w_{i} \geq 0$ for all $i$. Therefore, we have $\tilde \r  \tilde \r ^\top/\inner{\tilde \r }{\w }^2 \preceq \text{diag}(1/w_{1}^2, \dots, 1/w_{d-1}^2)$. Plugging this into \eqref{eq:lastt} implies that 
		\begin{align}
			\nabla f(\w) \nabla f(\w)^\top   \preceq  \text{diag}(1/w_{1}^2, \dots, 1/w_{d-1}^2) + (1-\mu)^{-2} \mathbf{1}\mathbf{1}^\top /\tilde w_{d}^2 \preceq (1-\mu)^{-2} \nabla^2 \Psi(\tilde \w).
		\end{align}
	\end{proof}
	
	\begin{lemma}
		\label{lem:gradsum}
		Let $\beta \leq 1/8$ and $\eta \in (0,1)$. If $T>1$, then the iterates $(\w_t)$ and $(\u_t)$ in Alg.~\ref{alg:DONSbase}, satisfy
		\begin{align}
			\sum_{t=1}^T \|\g_t\|^{2}_{\nabla^{-2}F_t(\w_t)} \leq \frac{16 d \ln T}{\beta},
		\end{align}
		where $(F_t)$ are as in \eqref{eq:mirrordescent} and $\g_t = \nabla \ell_t(\u_t)$.
	\end{lemma}
	\begin{proof}[{\bf Proof}]
		Let $c_T\coloneqq 1 - 1/T$.	First note that by Lemma \ref{lem:iterates}, we have $c_T^{2} \g_t\g_t^\top \preceq \nabla^2 \Psi(\w_t)$, where $\Psi(\x) \coloneqq - \sum_{i=1}^d \ln \bar x_i$. Combining this with the fact that $\eta_t \in[\eta, \eta e]$ and the range assumptions on $\beta$, $\eta$, and $T$, we get
		\begin{align}
			\|\g_t\|^{2}_{\nabla^{-2}F_t(\w_t)}=  \g_t^\top\left( \nabla^2 \Psi_t(\w_t) + \beta Q_{t-1}/4 \right)^{-1} \g_t  \leq  4\beta^{-1} \g_t^\top Q_{t}^{-1} \g_t. \label{eq:didit}
		\end{align}
		where $Q_t \coloneqq d I  +\sum_{s=1}^t \g_s\g_s^\top$. Thus, by \eqref{eq:didit} and \cite[Lemma 11]{hazan2007logarithmic}, we have
		\begin{align*}
			\|\g_t\|^{2}_{\nabla^{-2}F_t(\w_t)}\leq 	\frac{4}{\beta} \sum_{t=1}^T \g_t^\top Q_t^{-1} \g_t \leq \frac{4}{\beta}\ln \frac{|Q_T|}{|Q_0|} \leq \frac{4d\ln (1+T^3)}{\beta} \leq \frac{16d\ln T}{\beta},
		\end{align*}
		where the second inequality uses the fact $ \ln|Q_0|= d\ln d$ and by AM-GM inequality $ \ln|Q_T|\leq  d\ln \frac{\operatorname{Tr}(Q_T)}{ d} \leq  d\ln \left( d+{\sum_{t=1}^T \norm{\g_t}^2_2 }/{ d}\right)\leq  d\ln( d+ dT^3)$ since $$\norm{\g_t}^2_2 \leq  d^2T^2 \frac{\sum_{i=1}^d r_{t,i}^2}{\left(\sum_{i=1}^d r_{t,i}\right)^2} \leq  d^2T^2.$$
		This completes the proof.
	\end{proof}
	\begin{lemma}
		\label{lem:handylem}
		Suppose that $T>1$ and define $c_T \coloneqq 1 - 1/T$ and $\alpha_T \coloneqq 1+ \beta e \eta c_T^{-2}/4$. If $\eta,\beta \in(0,1)$, then the iterates $(\w_t)$ and gradients $(\g_t)$ in Algorithm \ref{alg:DONSbase}, and the regularizers $(\Phi_t)$ in \eqref{eq:regularizer} are such that, for all $t\in[T]$,
		\begin{align}
			\nabla^2 \Phi_t(\w_t) \preceq \alpha_T \nabla^2 \Phi_{t-1}(\w_t), \ \ \|\g_t\|^2_{\nabla^{-2} \Phi_t(\w_t)} \leq  \frac{e \eta}{c_T^2}, \ \ \text{and} \ \ \sum_{s=1}^t \|\g_t\|^2_{\nabla^{-2}\Phi_{t}(\w_t)} \leq \frac{16 d \ln T}{\beta}. 
		\end{align}
	\end{lemma}
	\begin{proof}[{\bf Proof}]
	By Lemma \ref{lem:iterates}, we have $\nabla^2 \Psi(\w_t)\succeq c_T^2\g_t\g_t^\top$, where $\Psi(\x) \coloneqq - \sum_{i=1}^d \ln \bar x_i$. Using this and the fact that $(\eta_{t,i})\subset [\eta, \eta e]$, we get that $\beta \g_t \g_t^\top/4 \leq \beta e\eta c_T^{-2} \nabla^2 \Psi_{t-1}(\w_t)/4 \preceq \beta e\eta c_T^{-2} \nabla^2\Phi_{t-1}(\w_t)/4$. Thus, adding $\nabla^2 \Psi_{t-1}(\w_t)$ on both sides and using that $\alpha_t =1+ \beta e \eta c_T^{-2}/4$, we get that 
	\begin{align}
	\alpha_T 	\nabla^2 \Phi_{t-1}(\w_t) \succeq \beta \g_t\g_t^\top/4 + \nabla^2 \Phi_{t-1}(\w_t) \succeq  \nabla^2 \Phi_t(\w_t),
		\end{align}	 
	where the last inequality follows by the fact that $\eta_{t,i}\geq \eta_{t-1,i}$, for all $i\in [d]$. The remaining inequalities follow from Lemmas \ref{lem:iterates} and \ref{lem:gradsum}.
	\end{proof}
	\begin{lemma}
		\label{lem:handy}
		Let $ \u,  \w \in \cC_{d-1}$. For any $\r\in [0,1]^d$ and $f(\x) \coloneqq -\log \inner{\r}{\bar \x}$, we have 
		\begin{align}
			|\inner{\nabla f(\w)}{\u -\w}| \leq 1\vee  \max_{i\in[d]}\frac{\bar u_{i}}{\bar w_i}.\nn
		\end{align}
	\end{lemma}
	\begin{proof}[{\bf Proof}]
		We have 
		\begin{align}
			\inner{\nabla f(\w)}{\u -\w} =\inner{\r}{\bar \w}^{-1} \inner{J^\top \r }{\u -
				\w} = \inner{\r}{\bar \w}^{-1} \inner{\r}{\bar \u -
				\bar \w}= \frac{\inner{\r}{\bar \u}}{\inner{\r}{\bar \w}} - 1.\nn
		\end{align}
		Now since the function $\r \mapsto \frac{\inner{\r}{\bar \u}}{\inner{\r}{\bar \w}}$ is quasi-convex \cite[Example 3.32]{boyd2004convex}, its maximum is reached at the boundary of $[0,1]^d$. Thus, the previous display implies that
		\begin{align}
			|\inner{\nabla f(\w)}{\u -\w}|\leq 1 \vee \sup_{\tilde\r\in[0,1]^d} \frac{\inner{\tilde \r}{\bar \u}}{\inner{\tilde \r}{\bar \w}} \leq 1 \vee \max_{i\in[d]} \frac{\bar u_i}{\bar w_i}.\nn
		\end{align}
	\end{proof}
	\section{Proof of Lemmas \ref{lem:inter0}, \ref{lem:selfconcord}, \ref{lem:close}, and \ref{lem:wom}}
	\label{sec:cproofs}
			\begin{proof}[{\bf Proof of Lemma \ref{lem:inter0}}]
		Since $\nabla f(\w) - \nabla f(\p) = \int_0^1   \nabla^2 f( \w_{\mu})(\w-\p) d\mu$, where $\w_{\mu}\coloneqq \mu \w+(1-\mu)\p$, we have,
		\begin{align}
			&	 \|\nabla f(\w) - \nabla f(\p)\|^2_{\nabla^{-2}f(\w)} \nn \\	=&  (\nabla f(\w) - \nabla f(\p))^\top\nabla^{-2}f(\w) (\nabla f(\w) - \nabla f(\p)),\nn \\
			= &(\w-\p)^\top \left( \int_{0}^1 \nabla^2 f( \w_{\mu})d\mu \right)^\top \nabla^{-2}f(\w)\left( \int_{0}^1 \nabla^2 f( \w_{\nu})d\nu \right) (\w-\p). \label{eq:ne}
		\end{align}
		On the other hand, since $f$ is self-concordant with constant $M_{f}$ and $r\coloneqq \|\p-\w\|_{\nabla^2 f(\w)}\leq 1/M_{f}$ by assumption, we have (see e.g.~\cite[Corollary 5.1.5]{nesterov2018lectures})
		\begin{align}
			(1- M_{f} r +1/3 M^2_{f} r^2)\nabla^2 f(\w)	 \preceq  H  \preceq \frac{1}{1-M_f r}\nabla^2 f(\w), \label{eq:sand}
		\end{align}
		where $H\coloneqq \int_{0}^1 \nabla^2 f( \w_{\mu})d\mu$. Since $\nabla^2 f(\w)$ is definite positive for all $\w$, \eqref{eq:sand} further implies that $H\succeq (1- M_{f} r )\nabla^2 f(\w)$. Combining these facts with \eqref{eq:ne}, we get that 
		\begin{align}
			\|\nabla f(\w) - \nabla f(\p)\|^2_{\nabla^{-2}f(\w)}  \preceq \frac{1}{(1-M_f r)^2}\|\p-\w\|^2_{\nabla^2 f(\w)}. \label{eq:newnew}
		\end{align}
	\end{proof}
			\begin{proof}[{\bf Proof of Lemma \ref{lem:selfconcord}}]
 We prove the claim by induction. We start with the base case $t=1$. For $t=1$, we have $\w_1 = \mathbf{1}/d\in \cC_{d-1}$. We now check that $\Phi_1$ is self-concordant with constant $\sqrt{e\eta}$. For any $i\in [d]$, the function $\vartheta_i\colon \x\rightarrow - \ln \bar x_i$ defined on $\cC_{d-1}$ is self-concordant with constant $1$. Furthermore, the function $\Theta_1\colon\u \mapsto \Phi_1(\u) - \Psi_1(\u)$ is self-concordant with constant $0$ (since it is a quadratic). Thus, by \cite[Theorem 5.1.1]{nesterov2018lectures} and the fact that $\Phi_1(\cdot)=\Theta_1(\cdot)+\Psi_1(\cdot)= \Theta_1(\cdot) + \sum_{i=1}^d \vartheta_i(\cdot)/ \eta_{1,i}$, we have that $\Phi_1$ is self-concordant with constant less than $0\vee \max_{i\in[d]} \sqrt{\eta_{1,i}} = \sqrt{\eta} \leq \sqrt{\eta e}$.
	
	Now, suppose the claim of the lemma holds for all $t\leq s$. We will show that it holds for $t=s+1$. Since $\Phi_s$ is self-concordant with constant $M_{\Phi_s} \leq \sqrt{\eta e}$ and $\w_s\in \cC_{d-1}$ (by the induction hypothesis), we have that the \emph{Dikin ellipsoid}  
	\begin{align}
		\cW_s \coloneqq \{ \x\in \reals^{d-1} \ : \ \|\x-\w_s\|_{\nabla^{2} \Phi_s(\w_s)} < 1/\sqrt{\eta e} \}
	\end{align}
	is contained within $\operatorname{dom}\Phi_s = \cC_{d-1}$ (see e.g.~\cite[Thm.~5.1.5]{nesterov2018lectures}). Thus, to show that $\w_{s+1}\in \cC_{d-1}$, it suffices to show that $\w_{s+1}\in \cW_{s}$. By definition of $\w_{s+1}$, we have  \begin{align}
		\|\w_{s+1} - \w_s\|_{\nabla^{2} \Phi_s(\w_s) } & = \frac{\|\nabla^{-2}\Phi_s(\w_s) \bm{\nabla}_s\|_{\nabla^2 \Phi_s(\w_s) }}{1+4\sqrt{\eta e} \|\bm{\nabla}_s\|_{\nabla^{2}\Phi_s(\w_s) }}   = \frac{\|\bm{\nabla}_s\|_{\nabla^2\Phi_s( \w_s)}}{1+ 4\sqrt{e \eta}\|\bm{\nabla}_s\|_{\nabla^2\Phi_s( \w_s)}} < \frac{1}{\sqrt{\eta e}},
	\end{align}
	where $\bm{\nabla}_s$ is as in \eqref{eq:damped}. This shows that $\w_{s+1}\in \cW_s$ and so $\w_{s+1}\in \cC_{d-1}$. As a consequence, the output $\u_{s+1}$ of Algorithm \ref{alg:DONSbase} satisfies $\bar u_{s+1,i}\geq 1/(dT)$, for all $i\in[d]$, which in turn implies that $\eta_{s+1,i} \in [\eta, e \eta]$, for all $i\in[d]$. Using this and \cite[Theorem~5.1.1]{nesterov2018lectures}, we have that $\Phi_{s+1}(\cdot) = \Phi_{s+1}(\cdot)-\Psi_{s+1}(\cdot)+ \sum_{i=1}^d \vartheta_i(\cdot)/\eta_{s+1,i}$ is self-concordant with constant less than $0 \vee \max_{i\in [d]} \sqrt{\eta_{s+1,i}}\leq \sqrt{\eta e}$.
\end{proof}
			\begin{proof}[{\bf Proof of Lemma \ref{lem:close}}]
		For any twice differentiable function $F\colon \cW \rightarrow \reals$ and $\w\in \cW$, we recall the definition of the Newton decrement $\lambda(\w,F)\coloneqq \|\nabla F(\w)\|_{\nabla^{-2}F(\w)}$ which will be useful in this proof. First, note that if $T=1$, then the result holds trivially since $\w_1=\p_1=\mathbf{1}/d$. Assume that $T> 1$ and let $c_T \coloneqq 1- 1/T$. Next, we will show by induction that
		\begin{align}
			\frac{1}{16\sqrt{e \eta}}\|\w_s-\p_s  \|_{\nabla^2 F_s(\w_s)} \leq \frac{1}{8 \sqrt{e\eta}} 	\lambda(\w_s, F_{s})\leq \lambda(\w_{s-1}, F_{s})^2 \leq C \eta, \label{eq:induction} 
		\end{align}
		for all $s \geq 1$, where $C\coloneqq 4e /c_T^2$ with the convention that $\w_0 =\mathbf{1}/d$. The base case follows trivially since $\nabla F_1(\w_0)=\nabla F_1(\w_1)=\bm{0}$ and $\w_1 =\p_1$. Suppose that \eqref{eq:induction} holds for $s=t$. We will show that it holds for $s=t+1$. By Lemma \ref{lem:wom}, we have $\nabla F_{t+1}(\w_t)=\g_t + \nabla F_t(\w_t)$, and so by the fact that $(a+b)^2 \leq 2 a^2 +2 b^2$, we get 
		\begin{align}
			\lambda(\w_t, F_{t+1})^2 &= \|\nabla F_{t+1}(\w_t)\|^2_{\nabla^{-2}F_{t}(\w_t)}  ,\nn \\ &\leq 2   \|\nabla F_t(\w_t)\|^2_{\nabla^{-2}F_{t}(\w_t)}  + 2 \|\g_t\|^2_{  \nabla^{-2}F_{t+1}(\w_t)},\nn \\
			& =2   \lambda(\w_t, F_t)^2  + 2 \|\g_t\|^2_{  \nabla^{-2}\Phi_{t}(\w_t)}, \label{eq:curz}\\
			& \le 2^7 e C^2\eta^3 + 2 e \eta/c_T^2 \leq  C \eta,
			\label{eq:postcurz} 
		\end{align}
		where in the penultimate inequality we used the induction hypothesis in \eqref{eq:induction} for $s=t$ and the bound on $\|\g_t\|^2_{  \nabla^{-2}\Phi_{t}(\w_t)}$ from Lemma \ref{lem:handylem}. The last inequality in \eqref{eq:postcurz} uses the range assumptions on $\eta$. Now, by the expression of $\nabla F_{t+1}(\w_t)$ in Lemma \ref{lem:wom}, one can verify that the iterate $\w_{t+1}$ in Algorithm \ref{alg:DONSbase} satisfies 
		\begin{align}
			\w_{t+1} = \w_t - \frac{1}{1 + 4\sqrt{e\eta} \lambda(\w_t,F_{t+1})} \nabla^{-2}F_{t+1}(\w_t)\nabla F_{t+1}(\w_t),  
		\end{align} 
		which is the damped Newton step with respect to the function $F_{t+1}$. Therefore, by Lemma~\ref{lem:properties} and the fact that $\lambda(\w_t,F_{t+1})\leq 1/(8\sqrt{e \eta})$ (which follows from \eqref{eq:postcurz} and the range assumption on $\eta$), we have $\lambda(\w_{t+1},F_{t+1}) \leq 8 \sqrt{e \eta}\lambda (\w_t,F_{t+1})^2$. Furthermore, since $\p_{t+1}$ is the minimizer of $F_{t+1}$ and $\lambda(\w_{t+1},F_{t+1})\leq 1/(2 \sqrt{e\eta})$, we have $\|\w_{t+1}- \p_{t+1}\|_{\nabla^2 F_{t+1}(\w_{t})}\leq 2 \lambda (\w_{t+1}, F_{t+1})$ (by Lemma~\ref{lem:properties} again). Combining these facts with \eqref{eq:postcurz}, implies \eqref{eq:induction} for $s=t+1$, which concludes the induction. This shows \eqref{eq:workingftrl}.
		
		We now use \eqref{eq:induction} together with \eqref{eq:curz} to bound the sum $\sum_{t=1}^T \|\w_t - \p_t\|_{\nabla^2 \Phi_{t-1}(\w_t)}^2$. Using that $\lambda(\w_{t+1},F_{t+1}) \leq 8\sqrt{e\eta}\lambda (\w_t,F_{t+1})^2$ (as argued above) and \eqref{eq:curz}, we get 
		\begin{align}
			\lambda(\w_{t+1}, F_{t+1}) \leq 16\sqrt{e\eta} 	\lambda(\w_{t}, F_{t})^2 + 16 \sqrt{e\eta} \|\g_t\|^{2}_{\nabla^{-2}\Phi_{t}(\w_t)}. \label{eq:match}
		\end{align}
		Summing \eqref{eq:match}, for $t=1,\dots, T$, rearranging, and using that $\lambda(\w_{T+1}, F_{T+1})\geq 0$, we get 
		\begin{align}
			\sum_{t=2}^T \left(\lambda(\w_t, F_t) - 16\sqrt{e\eta}  \lambda(\w_t, F_t)^2\right) \leq  16 \sqrt{e\eta}\lambda(\w_1, F_1)^2 + 16\sqrt{e\eta} \sum_{t=2}^T \|\g_t\|^{2}_{\nabla^{-2}\Phi_{t}(\w_t)}.
		\end{align}
		Using \eqref{eq:curz} and the range assumption on $\eta$, we have $0\leq 16\sqrt{e\eta}\lambda(\w_t,F_t)\leq  2^7 e C \eta^2 \leq 1/4$. Therefore, we have
		\begin{align}
			\frac{3}{4}\sum_{t=1}^T \lambda (\w_t,F_t) &\leq \lambda(\w_1,F_1) + 16 \sqrt{e\eta}\sum_{t=2}^T \|\g_t\|^{2}_{\nabla^{-2}\Phi_{t}(\w_t)},\nn \\ &  \leq  \frac{1}{16} + 16\sqrt{e\eta} \sum_{t=1}^T \|\g_t\|^{2}_{\nabla^{-2}\Phi_{t}(\w_t)}\leq  \frac{3}{8} +  \frac{45 d \ln  T}{8 \beta},
		\end{align}
		where the last inequality follows by Lemma \ref{lem:handylem} and the range assumption on $\eta$. Now, using the fact that $\p_t$ is the minimizer of $F_t$, we have $\|\w_t-\p_t\|_{\nabla^{2}F_t(\w_t)}\leq 2 \lambda(\w_t,F_t)$. Combining this with the previous display, we get the desired result. 
	\end{proof}
			\begin{proof}[{Proof \bf of Lemma \ref{lem:wom}}]
	Let $t\in [T]$, $\vartheta_i\colon \x \mapsto - \ln \bar x_i$, and $\delta_{t,i}\coloneqq (1/\eta_{t,i}-1/\eta_{t-1,i})$ with $\bm{\eta}_{0}=\eta \mathbf{1}$, for all $i\in[d]$. We will first show that $\nabla F_{t+1}(\w) = \tilde\g_t(\w) + \nabla F_t(\w)$,
	where $$\tilde \g_t(\w) \coloneqq  \g_t \cdot (1+\beta \inner{\g_t}{\w-\w_t}/4)+ \sum_{i=1}^d \delta_{t,i} (\nabla \vartheta_i(\w)- \nabla \vartheta_i(\w_t)).$$ 
	By definition of $(F_{t})$ and $(\p_t)$, we have
	\begin{align}
		\nabla F_{t+1}(\w) & = \tilde \g_t + \nabla \Phi_t(\w) - \nabla \Phi_t(\p_t),\label{eq:firstone} \\ & =  \tilde \g_t + {\beta} \g_t \cdot( \inner{\g_t}{\w-\w_t} -  \inner{\g_t}{\p_t-\w_t})/4  +\sum_{i=1}^d \delta_{t,i}(\nabla \vartheta_i(\w) - \nabla \vartheta_i(\p_t))  \nn \\ & \quad   + \nabla \Phi_{t-1}(\w) - \nabla \Phi_{t-1}(\p_t)  ,\nn \\ \displaybreak[0] & = \g_t \cdot (1+\beta\inner{\g_t}{\w- \w_t}/4) + \sum_{i=1}^d \delta_{t,i}(\nabla \vartheta_i(\w) - \nabla \vartheta_i(\w_t))\nn \\ & \quad  + \g_{t-1}+ \nabla \Phi_{t-1}(\w) - \nabla \Phi_{t-1}(\p_{t-1}) ,  \label{eq:break0}\\ \displaybreak[0]
		& = \g_t \cdot (1+\beta\inner{\g_t}{\w- \w_t}/4)+ \sum_{i=1}^d \delta_{t,i}(\nabla \vartheta_i(\w) - \nabla \vartheta_i(\w_t)) +\nabla F_{t}(\w), \label{eq:oldnew0}
	\end{align} 
	where \eqref{eq:break0} follows by definition of $\tilde \g_t$ in \eqref{eq:tildedgt} and the fact that $\bm{0}=\nabla F_{t}(\p_t)= \g_{t-1}+ \nabla \Phi_{t-1}(\p_t) -\nabla \Phi_t(\p_{t-1})$. Setting $\w= \w_t$ in \eqref{eq:oldnew0} shows the first equality of the lemma. Now, by induction, we get 
	\begin{align}
		\nabla F_{t+1}(\w) & = \sum_{s=1}^t \g_s \cdot (1+\beta\inner{\g_s}{\w- \w_s}/4)+ \sum_{s=1}^t \sum_{i=1}^d \delta_{s,i}(\nabla \vartheta_i(\w) - \nabla \vartheta_i(\w_s)) +\nabla F_1(\w),\nn \\
		& =  \sum_{s=1}^t \g_s \cdot (1+\beta\inner{\g_s}{\w- \w_s}/4)+ \nabla \Psi_t(\w) - \sum_{s=1}^t \sum_{i=1}^d \delta_{s,i} \vartheta_i(\w_s) + \beta d \w/4.
	\end{align}
	Using that $\nabla \Phi_t(\w) =\beta d \w/4+ \sum_{s=1}^t \beta \g_s \cdot \inner{\g_s}{\w- \w_s}/4+ \nabla \Psi_t(\w)$ completes the proof.
\end{proof}
	
	\section{Proof of Theorem \ref{lem:thekey0} (Regret of Base Algorithm)}
	\label{sec:proofs0}
	We present the proof of Theorem \ref{lem:thekey0} before proving Lemmas \ref{lem:decomp}-\ref{lem:pseudostab}.	
			\begin{proof}[{\bf Proof of Theorem \ref{lem:thekey0}}]
		Our starting point is the regret decomposition in Lemma \ref{lem:decomp}. Using Lemmas~\ref{lem:concord}, \ref{lem:stabi}, and \ref{lem:decomplemnew} to bound the first two sums on the RHS of the regret decomposition \eqref{eq:trip}, we get 
		\begin{align}
			\sum_{t=1}^T (\ell_t(\u_t)  - \ell_t(\u''))& \leq \frac{18 d \ln T}{\beta} +O\left(\frac{d \ln T}{\eta}\right)- \frac{1}{32 \eta \ln T} \sum_{i=1}^d \max_{t\leq T}  \frac{\bar u''_{i}}{\bar u_{t,i}}	 \nn \\
			& \quad +c_T \sum_{t=1}^T \inner{ \g_t}{\w_t - \p_t}+ \frac{3 \beta}{8 } \sum_{t=1}^T\inner{\g_t}{\w_t-\p_t}^2. \label{eq:think}
		\end{align}
		Now using Lemma \ref{lem:pseudostab} to bound the last two sums in \eqref{eq:think}, we get the desired result. 
	\end{proof}
 	In addition to Lemmas \ref{lem:decomp}-\ref{lem:pseudostab} in the main body, we also need the following result (which follows from the proof of \cite[Lemma~5]{luo2018efficient}) to prove Theorem \ref{lem:thekey0}:
	\begin{lemma}
		\label{lem:neagtiveterm}
		Let $(\p_t)$ and $(\tilde \g_t)$ be as in \eqref{eq:mirrordescent} and \eqref{eq:tildedgt}, respectively. Then, $\forall t\in[T],\forall \u \in \cC_{d-1}$,
		\begin{align}
			\sum_{t=1}^T	\inner{\tilde \g_t}{\p_{t+1}-\u'} & \leq O\left( \frac{d\ln T}{\eta}\right) + \sum_{t=1}^T \left(\frac{\beta}{8} \inner{\g_t}{\p_t - \u'}^2 + D_{\Psi_t}(\u',\p_t) -D_{\Psi_{t-1}}(\u',\p_{t})\right).   \label{eq:secondtarget}
		\end{align} 
	\end{lemma}
	
	\begin{proof}[{\bf Proof}]
		In this proof, we let $\Theta_t\coloneqq \Phi_t - \Psi_t$, for all $t\geq 1$. Since $\p_{t+1}$ is the minimizer of $F_{t+1}$, which is a self-concordant barrier for the set $\cC_{d-1}$, we have $\bm{0}=\nabla F_{t+1}(\p_{t+1}) = \tilde \g_t + \nabla \Phi_t(\p_{t+1}) -\nabla \Phi_{t}(\p_t)$. Therefore, we have 
		\begin{align}
			\inner{\tilde \g_t}{\p_{t+1}-\u'} & \leq \inner{\nabla \Phi_t(\p_{t+1}) - \nabla \Phi_t( \p_t)}{\u' - \p_{t+1}},\nn \\
			& = D_{\Phi_t}(\u',  \p_t) - D_{\Phi_t}(\u', \p_{t+1}) - D_{\Phi_t}(\p_{t+1}, \p_t),\label{eq:bregman}\\
			& \leq D_{\Phi_t}(\u',  \p_t) - D_{\Phi_t}(\u', \p_{t+1}), \label{eq:positivity}\\ 
			& = D_{\Psi_t}(\u',  \p_t) - D_{\Psi_t}(\u', \p_{t+1}) + D_{\Theta_t}(\u',  \p_t) - D_{\Theta_t}(\u', \p_{t+1}), \label{eq:pano}
		\end{align}
		where \eqref{eq:bregman} follows by the definition of the Bregman divergence, and \eqref{eq:positivity} follows by the positivity of the Bregman divergence. Summing \eqref{eq:pano} for $t=1$ to $T$, we get 
		\begin{align}
			\sum_{t=1}^T \inner{\tilde \g_t}{\p_{t+1}-\u} &\leq  D_{\Psi_0}(\u', \w_1) + D_{\Theta_0}(\u', \w_1) + \sum_{t=1}^T (D_{\Theta_t}(\u',  \p_t) -D_{\Theta_{t-1}}(\u',  \p_t)) \nn \\
			& \quad +  \sum_{t=1}^T (D_{\Psi_t}(\u',  \p_t) -D_{\Psi_{t-1}}(\u',  \p_t)),\nn \\
			&=  D_{\Psi_0}(\u', \w_1) + D_{\Theta_0}(\u', \w_1) + \sum_{t=1}^T \frac{\beta}{8}\inner{\g_t}{\p_t-\u'}^2 \nn \\
			& \quad +  \sum_{t=1}^T (D_{\Psi_t}(\u',  \p_t) -D_{\Psi_{t-1}}(\u',  \p_t)). \label{eq:piano}
		\end{align}
		By definition of $\u'$ and $\w_1$, we have $D_{\Psi_0}(\u', \w_1) + D_{\Theta_0}(\u', \w_1)  \leq O(\eta^{-1}d \ln T )$, which combined with \eqref{eq:piano} implies the desired result.
	\end{proof}

		\begin{proof}[{\bf Proof of Lemma \ref{lem:concord}}]
		Denote by $\cT\in [T]$ the subset of rounds $t$ where any of $(\rho_{t,i})_{i\in[d]}$ change. For $t\not\in \cT$, we have $\psi_t\equiv 0$, which is self-concordant with any constant. Now, let $t\in\cT$. Since $\psi_t$ is the sum of self-concordant functions, we have by \cite[Theorem~5.1.1]{nesterov2018lectures} that $\psi_t$ is a self-concordant function with constant less than 
		\begin{align}\max_{i\in[d]} \sqrt{\left(\frac{1}{\eta_{t-1,i}} -\frac{1}{\eta_{t,i}}  \right)^{-1}}&=
			\max_{i\in[d]}  	\sqrt{\frac{\eta_{t,i}}{e^{\log_T(\rho_{t,i}/\rho_{t-1,i})}-1}}\nn \\  & \stackrel{	(*)}{\leq} \max_{i\in[d]}\sqrt{\frac{e\eta}{\log_T(\rho_{t,i}/\rho_{t-1,i})}}\stackrel{	(**)}{\leq}  \sqrt{e\eta  \log_2 T},
		\end{align}
		where $(*)$ follows by the fact that $e^x -1 \geq x$, for all $x\in \reals$, and that $\rho_{t,i}> 2\rho_{t-1,i}$ since $t\in \cT$. This shows the first claim of the lemma. We now show the second claim. Let $t\in \cT$ and define $\w_{t,\mu} \coloneqq \mu \p_t + (1-\mu) \w_t$, for $\mu \in[0,1]$. By Lemmas \ref{lem:hessians}, \ref{lem:selfconcord}, and  \ref{lem:close}, we have 
		\begin{align}(1- \sqrt{e\eta} r_t)^{2}\nabla^{-2}\Phi_{t-1}(\w_{t,\mu})\preceq  \nabla^{-2}\Phi_{t-1}(\w_t) \preceq (1- \sqrt{e\eta} r_t)^{-2}\nabla^{-2}\Phi_{t-1}(\w_{t,\mu}),
			\label{eq:gentily}
		\end{align}
		where $r_t \coloneqq \|\w_t-\p_t\|_{\nabla^2 \Phi_{t-1}(\w_t)}$. By Taylor's theorem, there exists $\mu_* \in[0,1]$ such that 
		\begin{align}
			D_{\psi_t}(\p_t,\w_t) & = \frac{1}{2}\|\p_t-\w_t\|^2_{\nabla^2 \psi_t(\w_{t,\mu_*})} \leq \frac{1}{2}\|\p_t-\w_t\|^2_{\nabla^2  \Phi_{t-1}(\w_{t,\mu_*})} \label{eq:clear}, \\
			& \leq \frac{1}{2(1-\sqrt{e\eta} r_t)^2}\|\p_t-\w_t\|^2_{\nabla^2  \Phi_{t-1}(\w_{t})}.\label{eq:down}
		\end{align}
		where \eqref{eq:clear} follows by the fact that $\nabla^2 \psi_{t}\preceq \nabla^2 \Psi_{t-1} \preceq \nabla^2 \Phi_{t-1}$ and the last inequality follows by \eqref{eq:gentily}. Plugging the bound on $r_t =\|\p_t-\w_t\|^2_{\nabla^2  \Phi_{t-1}(\w_{t})}\leq 64 (e \eta)^{3/2}/c_T^2$ (from Lemma \ref{lem:close}) into \eqref{eq:down} and using the facts that \textbf{I)} $|\cT|\leq O(d \ln T)$ (by definition of $(\rho_{t,i})$ in \eqref{eq:choice} and the fact that $(\bar u_{t,i})\subset [1/(dT),1]$); and \textbf{II)} $D_{\psi_t}(\p_t,\w_t)=0$ if $t\not\in \cT$, we get that $\sum_{t=1}^TD_{\psi_t}(\p_t,\w_t) \leq O(\eta^{3/2} d \ln T)\leq O(d \ln T)$.
	\end{proof}

		\begin{proof}[{\bf Proof of Lemma \ref{lem:decomp}}]
		Let $c_T \coloneqq 1 - 1/T$ and $\u$ be as in the lemma's statement and recall the definition of $(\tilde \g_t)$ from \eqref{eq:tildedgt}. First, we note that by Lemma~\ref{lem:additive}, we have $ \sum_{t=1}^T \ell_t(\u) \leq \sum_{t=1}^T \ell_t(\u'')+4$, and so it suffices to bound the regret against $\u''$. Let $\hat \g_t \coloneqq \g_t (1+\beta \inner{\g_t}{\p_t- \w_t}/4)$. Using that Cover's loss is $1$-exp-concave and $\beta \leq 8^{-1} \wedge |8 \inner{\g_t}{\u_t-\u''}|^{-1}$, for all $t\geq 1$, we have (see e.g.~proof of \cite[Lemma~5]{luo2018efficient} for the first inequality)
		\begin{align}
			& \sum_{t=1}^T (\ell_t(\u_t)  - \ell_t(\u''))\nn \\ \leq &  \sum_{t=1}^T \inner{\g_t}{\u_t - \u''} - \frac{\beta}{2}\sum_{t=1}^T \inner{\g_t}{\u_t - \u''}^2,\nn \\\displaybreak[1]
			\leq  & \sum_{t=1}^T \inner{\hat \g_t}{\u_t - \u''} + \frac{\beta}{4} \sum_{t=1}^T|\inner{\g_t}{\w_t-\p_t} \inner{\g_t}{\u_t- \u''}| - \frac{\beta}{2}\sum_{t=1}^T \inner{\g_t}{\u_t - \u''}^2,\nn \\ \displaybreak[1]
			\leq &  \sum_{t=1}^T \inner{\hat \g_t}{\u_t - \u''} + \frac{\beta}{8 } \sum_{t=1}^T\inner{\g_t}{\w_t-\p_t}^2 - \frac{3\beta }{8} \sum_{t=1}^T \inner{\g_t}{\u_t- \u''}^2,\nn \\ \displaybreak[1]
			= & c_T\sum_{t=1}^T \inner{\hat  \g_t}{\w_t - \u'} +  \frac{\beta}{8 } \sum_{t=1}^T\inner{\g_t}{\w_t-\p_t}^2- \frac{3c^2_T\beta}{8}\sum_{t=1}^T \inner{\g_t}{\w_t - \u'}^2,\nn \\ \displaybreak[1]
			= & c_T \sum_{t=1}^T \inner{\hat \g_t}{\w_t - \p_t}+ \frac{\beta}{8 } \sum_{t=1}^T\inner{\g_t}{\w_t-\p_t}^2-\frac{3c_T^2 \beta}{8} \sum_{t=1}^T \inner{\g_t}{\w_t - \u'}^2  \nn \\ & \quad + c_T\sum_{t=1}^T \inner{\hat \g_t}{\p_t - \u'} ,\nn \\ \displaybreak[1]
			\leq  & c_T \sum_{t=1}^T \inner{\g_t}{\w_t - \p_t}-\frac{3c_T^2 \beta}{8} \sum_{t=1}^T \inner{\g_t}{\w_t - \u'}^2 + c_T\sum_{t=1}^T \inner{\hat \g_t}{\p_t - \u'}, 	 \label{eq:laststs}
		\end{align}
		where in the last inequality we used the fact that $\inner{\hat \g_t}{\w_t-\p_t}= \inner{\g_t}{\w_t-\p_t}- \beta \inner{\g_t}{\w_t-\p_t}^2/4$ and $c_T \geq 1/2$. We now focus on the last sum in \eqref{eq:laststs}. By definition of $(\tilde \g_t)$ and $(\hat \g_t)$, we have 
		\begin{align}
			\sum_{t=1}^T \inner{\hat \g_t}{\p_t - \u'}& =	\sum_{t=1}^T \inner{\tilde \g_t}{\p_t - \u'} \\ & \quad + \sum_{t=1}^T\sum_{i=1}^d\left(\frac{1}{\eta_{t,i}} - \frac{1}{\eta_{t-1,i}}\right) \inner{\nabla \vartheta_i(\w_t) - \nabla \vartheta_i(\p_t)}{\p_t -\u'},\nn \\ \displaybreak[0]
			&=\sum_{t=1}^T \sum_{i=1}^d \left(\frac{1}{\eta_{t,i}} - \frac{1}{\eta_{t-1,i}}\right) (D_{\vartheta_i}(\u',\w_t) -D_{\vartheta_i}(\u',\p_t)- D_{\vartheta_i}(\p_t,\w_t)) \\ & \quad + \sum_{t=1}^T \inner{\tilde \g_t}{\p_t - \p_{t+1}} +\sum_{t=1}^T \inner{\tilde \g_t}{\p_{t+1} - \u'}  ,\label{eq:defbreg} \\ \displaybreak[0]
			&= \sum_{t=1}^T \inner{\tilde \g_t}{\p_t - \p_{t+1}} +\sum_{t=1}^T \inner{\tilde \g_t}{\p_{t+1} - \u'}  \\ & \quad    + \sum_{t=1}^T \sum_{i=1}^d \left(\frac{1}{\eta_{t-1,i}} - \frac{1}{\eta_{t,i}}\right)  D_{\vartheta_i}(\p_t, \w_t)\nn \\ & \quad +  \sum_{t=1}^T (D_{\Psi_t}(\u',\w_t)- D_{\Psi_{t-1}}(\u',\w_t) - D_{\Psi_t}(\u',\p_t)+ D_{\Psi_{t-1}}(\u',\p_t)).
		\end{align}
		where \eqref{eq:defbreg} uses the definition of the Bregman divergence. Plugging in the bound on $\sum_{t=1}^T \inner{\tilde \g_t}{\p_{t+1} - \u'}$ from Lemma \ref{lem:neagtiveterm} and letting $\psi_t \colon \p \mapsto \sum_{i=1}^d \left(\frac{1}{\eta_{t-1,i}} - \frac{1}{\eta_{t,i}}\right)  \vartheta_i(\p)$, we get 
		\begin{align}
			\sum_{t=1}^T \inner{\hat \g_t}{\p_t - \u'}& \leq \sum_{t=1}^T \inner{\tilde \g_t}{\p_t - \p_{t+1}}  + \sum_{t=1}^T D_{\psi_t}(\p_t, \w_t)+  \sum_{t=1}^T (D_{\Psi_t}(\u',\w_t)- D_{\Psi_{t-1}}(\u',\w_t))  \nn \\ & \quad +O\left(\frac{d \ln T}{\eta} \right) +\sum_{t=1}^T \frac{\beta}{8} \inner{\g_t}{\u'-\p_t}^2.\label{eq:volk}
		\end{align}
		Now, by the fact $c_T \geq 1/2$ and that $(a+b)^2 \leq 3/2 a^2 + 3 b^2$, for all $T\geq 1$, we get 
		\begin{align}
			\sum_{t=1}^T \frac{\beta}{8} \inner{\g_t}{\u'-\p_t}^2 \leq \frac{3c_T \beta }{8} \inner{\g_t}{\u' - \w_t}^2 + \frac{3 \beta }{8} \inner{\g_t}{\p_t - \w_t}^2.  
		\end{align}
		By combining this with \eqref{eq:volk}, \eqref{eq:laststs}, and the fact that $\sum_{t=1}^T \ell_t(\u) \leq \sum_{t=1}^T \ell_t(\u'')+4$ (see beginning of the proof), we obtain the desired result.
	\end{proof}

			\begin{proof}[{\bf Proof of Lemma \ref{lem:stabi}}]
		Let $(F_{t})$ be as in \eqref{eq:mirrordescent} and $\alpha_T \coloneqq 1 + \beta e \eta c_T^{-2}/4$, where $c_T\coloneqq 1-1/T$. We start by bounding the Newton decrement $\lambda(\p_t,F_{t+1})=\|\nabla F_{t+1}(\p_t) \|_{\nabla^{-2}F_{t+1}(\p_t)}$. Since $F_{t+1}$ is equal to a linear function plus $\Phi_{t}$, Lemma \ref{lem:selfconcord} implies that $F_{t+1}$ is $\sqrt{e \eta}$-self-concordant. On the other hand, by Lemmas \ref{lem:handylem} and Lemma \ref{lem:close}, we have \begin{align} \|\w_t - \p_t\|_{\nabla^2 \Phi_{t}(\w_t)} \leq \sqrt{\alpha_T}\|\w_t-\p_t\|_{\nabla^2 \Phi_{t-1}(\w_t)} \leq \frac{2^6\sqrt{\alpha_T}(e \eta)^{3/2}}{c_T^2}<\frac{1}{2\sqrt{e\eta}},\label{eq:attn}
		\end{align} where the last inequality follows by the range assumptions on $\eta$ and $\beta$. This, together with the fact that $F_{t+1}$ is $\sqrt{e \eta}$-self-concordant and Lemma \ref{lem:inter0}, we have 
		\begin{align}
			\|\nabla F_{t+1}(\p_t) - \nabla F_{t+1}(\w_t)\|_{\nabla^{-2}F_{t+1}(\w_t)}& \leq 	\frac{ \|\w_t-\p_t\|_{\nabla^2 \Phi_{t}(\w_t)}}{1- \sqrt{e\eta}\|\w_t-\p_t\|_{\nabla^2 \Phi_{t}(\w_t)}} \leq \frac{2^7\sqrt{\alpha_T}}{c_T^2(e \eta)^{-3/2}}. \label{eq:latin}
		\end{align}
		Thus, by the triangle inequality, we get that 
		\begin{align}
			\|\nabla F_{t+1}(\p_t)\|_{\nabla^{-2}F_{t+1}(\w_t)} &\leq \|\nabla F_{t+1}(\w_t)\|_{\nabla^{-2}F_{t+1}(\w_t)}  + \|\nabla F_{t+1}(\p_t) - \nabla F_{t+1}(\w_t)\|_{\nabla^{-2}F_{t+1}(\w_t)},\nn \\
			& =  \lambda(\w_t, F_{t+1})  + \|\nabla F_{t+1}(\p_t) - \nabla F_{t+1}(\w_t)\|_{\nabla^{-2}F_{t+1}(\w_t)},\nn \\
			& \leq 2\sqrt{e\eta}/c_T +2^7\sqrt{\alpha_T}(e \eta)^{3/2}/c_T^2,
		\end{align}
		where the last inequality follows by \eqref{eq:latin} and Lemma \ref{lem:close}. Now, by \eqref{eq:attn} and Lemma~\ref{lem:hessians}, we have \begin{align}(1- \sqrt{e\eta} r_t)^{2}\nabla^{-2}F_{t+1}(\p_t)\preceq  \nabla^{-2}F_{t+1}(\w_t) \preceq (1- \sqrt{e\eta} r_t)^{-2}\nabla^{-2}F_{t+1}(\p_t),
			\label{eq:gentil}
		\end{align}
		where $r_t\coloneqq \|\w_t-\p_t\|_{\nabla^{2}\Phi_t(\w_t)}$, and so
		\begin{align}
			\|\nabla F_{t+1}(\p_t)\|_{\nabla^{-2}F_{t+1}(\p_t)}  & \leq (1-\sqrt{e\eta}\|\w_t-\p_t\|_{\nabla^{2}\Phi_t(\w_t)} )^{-1}  \|\nabla F_{t+1}(\p_t)\|_{\nabla^{-2}F_{t+1}(\w_t)},\nn \\
			& \leq 4\sqrt{e\eta}/c_T +2^8\sqrt{\alpha_T}(e \eta)^{3/2}/c_T^2\leq \frac{1}{2\sqrt{e \eta}}, \label{eq:dad}
		\end{align}
	where the last inequality follows by the range assumption on $\eta$ and $\beta$. Combining \eqref{eq:dad} with the fact that $\p_{t+1}$ is the minimizer of $F_{t+1}$ (which is $\sqrt{e\eta}$-self-concordant as we argued above) and Lem.~\ref{lem:properties}, we get \begin{align}\tilde r_t \coloneqq \|\p_t-\p_{t+1}\|_{\nabla^2 F_{t+1}(\p_t)}\leq 2 \lambda(\p_t,F_{t+1})\leq \frac{8\sqrt{e\eta}}{c_T} +\frac{2^9\sqrt{\alpha_T}(e \eta)^{3/2}}{c_T^2}<\frac{1}{2\sqrt{e\eta}},\label{eq:boston}
		\end{align} where the last inequality follows by the range assumption on $\eta$ and $\beta$. Thus, using  Lemma~\ref{lem:hessians} again, we get that \begin{align}(1-\sqrt{e\eta} \tilde r_t)^{2}\nabla^2 F_{t+1}(\p_t) \preceq  \nabla^2 F_{t+1}(\p_{t+1})\preceq (1-\sqrt{e\eta} \tilde r_t)^{-2} \nabla^2 F_{t+1}(\p_{t}).\label{eq:club}
		\end{align} Using this and \eqref{eq:boston}, we have 
		\begin{align}
			\|\p_t-\p_{t+1}\|_{\nabla^2 F_{t+1}(\p_{t+1})}\leq\frac{	\|\p_t-\p_{t+1}\|_{\nabla^2 F_{t+1}(\p_{t})}}{1-\sqrt{e\eta} \tilde r_t} \leq  \frac{16\sqrt{e\eta}}{c_T} +\frac{2^{14}\sqrt{\alpha_T}(e \eta)^{3/2}}{c_T^2}. \label{eq:drink}
		\end{align}
		The RHS of \eqref{eq:drink} is at most $1/2$ due to the range assumptions on $\eta$ and $\beta$. This, combined with \cite[Theorem 5.1.8]{nesterov2018lectures} and the optimality of $\p_{t+1}$ implies that for $\omega(x)\coloneqq x - \ln (1+x)$, we have
		\begin{align*}
			F_{t+1}(\p_t) - F_{t+1}(\p_{t+1}) &\geq   \nabla F_{t+1}(\p_{t+1})^\top(\p_t - \p_{t+1}) + 	\omega(\|\p_t-\p_{t+1}\|_{\nabla^2 F_{t+1}(\p_{t+1})}), \\
			& \geq  \frac{1}{3}\|\p_t-\p_{t+1}\|^2_{\nabla^2 F_{t+1}(\p_{t+1})}, 
		\end{align*}
		where the last inequality follows by the fact that $\omega(x)\geq x^2/3$, for $x\in[0,1/2]$. On the other hand, by the definition of $F_{t+1}$, non-negativity of Bregman divergence, and H\"{o}lder inequality, 
		\begin{align*}
			F_{t+1}(\p_t) - F_{t+1}(\p_{t+1}) & =  \inner{\p_t-\p_{t+1}}{ \tilde \g_t} - D_{\Phi_t}(\p_{t+1}, \p_t) ,
			\nn \\ & \leq  \norm{\p_t - \p_{t+1}}_{\nabla^2F_{t+1}(\p_{t+1})} \cdot \norm{\tilde \g_t}_{\nabla^{-2}F_{t+1}(\p_{t+1})}.
		\end{align*}
		Combining the above two inequalities we get \begin{align}
			\norm{\p_t - \p_{t+1}}_{\nabla^2F_{t+1}(\p_{t+1})} \leq 3\norm{\tilde \g_t}_{\nabla^{-2}F_{t+1}(\p_{t+1})}.
		\end{align}
		Using this and H\"older's inequality leads to 
		\begin{align}
			\inner{\tilde \g_t}{ \p_t - \p_{t+1}} \leq \norm{\tilde \g_t}_{\nabla^{-2}F_{t+1}(\p_{t+1})}\cdot \norm{\p_t - \p_{t+1}}_{\nabla^{2}F_{t+1}(\p_{t+1})} &\leq 3\norm{\tilde \g_t}_{\nabla^{-2}F_{t+1}(\p_{t+1})}^2,\nn \\ &\leq  4\norm{\tilde \g_t}_{\nabla^{-2}F_{t+1}(\w_t)}^2, \label{eq:lasty}
		\end{align}
		where the last inequality follows by \eqref{eq:club}, \eqref{eq:gentil}, and the range assumptions on $\eta$ and $\beta$. Let $\psi_t \coloneqq \Psi_{t-1}-\Psi_t$ and recall the definition of $(\vartheta_i)$ in \eqref{eq:thepsi}. Let $\psi_t\coloneqq \Psi_{t-1}-\Psi_t$. By the definition of $\tilde \g_t$ in \eqref{eq:tildedgt}, the fact that $(a+b)^2 \leq (1+\gamma) a^2 +(1+1/\gamma) b^2$, for all $\gamma>0$, and \eqref{eq:lasty}, we have
		\begin{align}
			\inner{\tilde \g_t}{ \p_t - \p_{t+1}} 	& \leq   {9}(1+ \beta \inner{\g_t}{\p_t - \w_t}/4)^2 \cdot  \|\g_t\|^2_{\nabla^{-2}F_{t+1}(\w_t)}/8  \nn \\
			&  \quad + 9\left\|\sum_{i=1}^d \left(\frac{1}{\eta_{t-1,i}} - \frac{1}{\eta_{t,i}}\right) (\nabla \vartheta_i(\p_t) -  \nabla \vartheta_i(\w_t))\right\|^2_{\nabla^{-2} F_{t+1}(\w_t)},\nn \\
			&\leq    {9}(1+ \beta \|\g_t\|_{\nabla^{-2} \Phi_{t-1}(\w_t)}\cdot \|\p_t - \w_t\|_{\nabla^2 \Phi_{t-1}(\w_t)}/4)^2   \cdot \|\g_t\|^2_{\nabla^{-2}\Phi_{t}(\w_t)}/8 \nn \\ & \quad  + 9e\|\nabla \psi_t (\p_t)- \nabla \psi_t(\w_t) \|^2_{\nabla^{-2} \psi_{t}(\w_t)},\label{eq:ride}
		\end{align}
		where in the last inequality we used H\"older's inequality and the fact that $e\nabla^2 F_{t+1}\succeq \nabla^2  \Psi_{t-1}  \succeq\nabla^2 \psi_{t}$ (since $\eta_{t,i}\in [\eta, \eta e]$, for all $i\in[d]$ and $t\in[T]$). On the other hand, since $\nabla^2 \psi_t\preceq \nabla^2  \Psi_{t-1}\preceq \nabla^2 \Phi_{t-1}$, we have $\|\w_t-\p_t\|^2_{\nabla^2 \psi_t(\w_t)}\leq \|\w_t-\p_t\|^2_{\nabla^2 \Phi_{t-1}(\w_t)}$. Using this, \eqref{eq:attn}, and Lemmas \ref{lem:concord} and \ref{lem:inter0}, we get
		\begin{align}
			\|\nabla \psi_t (\p_t)- \nabla \psi_t(\w_t) \|^2_{\nabla^{-2} \psi_{t}(\w_t)} &\leq\mathbb{I}\{\rho_{t,i} \neq \rho_{t-1,i}\} \frac{  \|\w_t - \p_t\|^2_{\nabla^2 \psi_t(\w_t)}}{(1- \sqrt{e\eta \log_2 T} \|\w_t - \p_t\|_{\nabla^2 \psi_t(\w_t)})^2},\nn \\ &\leq  \mathbb{I}\{\rho_{t,i} \neq \rho_{t-1,i}\} \frac{  \|\w_t - \p_t\|^2_{\nabla^2  \Phi_{t-1}(\w_t)}}{(1- \sqrt{e\eta \log_2 T} \|\w_t - \p_t\|_{\nabla^2  \Phi_{t-1}(\w_t)})^2}, \nn \\ &  \leq  2^{14} \mathbb{I}\{\rho_{t,i} \neq \rho_{t-1,i}\} (e \eta)^3/c_T^4 ,\label{eq:cholate}
		\end{align}
		where the last two inequalities follow by \eqref{eq:attn} and the range assumptions on $\eta$ and $\beta$. Since $\sum_{t=1}^T \mathbb{I}\{\rho_{t,i} \neq \rho_{t-1,i}\}\leq O(d \ln T)$ (by definition of $(\rho_{t,i})$ in \eqref{eq:choice} and the fact that $(\bar u_{t,i})\subset [1/(dT),1]$), \eqref{eq:cholate} implies that $\sum_{t=1}^T\|\nabla \psi_t (\p_t)- \nabla \psi_t(\w_t) \|^2_{\nabla^{-2} \psi_{t}(\w_t)} \leq O(\eta^3 d\ln T)\leq O(d\ln T)$. Using this together with \eqref{eq:attn}, \eqref{eq:ride}, and Lemma \ref{lem:handylem} (and the range assumptions on $\eta$ and $\beta$), we get 
		\begin{align}
			\sum_{t=1}^T 	\inner{\tilde \g_t}{ \p_t - \p_{t+1}}  \leq \frac{18 d \ln T}{\beta} + O(d\ln T). 
		\end{align} 
		This completes the proof.
	\end{proof}
			\begin{proof}[{\bf Proof of Lemma \ref{lem:pseudostab}}]
		Let $c_T \coloneqq 1-1/T$. By Cauchy Schwarz inequality and Lemmas \ref{lem:close} and \ref{lem:gradsum}, 
		\begin{align}
			\sum_{t=1}^T \inner{\g_t}{\p_t- \w_t} \leq 	\sum_{t=1}^T |\inner{\g_t}{\p_t- \w_t}|  & \leq \sqrt{\sum_{t=1}^T  \|\g_t\|^{2}_{\nabla^{-2}\Phi_{t-1}(\w_t)}}  \sqrt{\sum_{t=1}^T  \|\p_t-\w_{t}\|^{2}_{\nabla^2 \Phi_{t-1}(\w_t)}},\nn \\
			& \leq \sqrt{ 4^2d\beta^{-1}\ln T} \cdot \sqrt{1+  15 d\beta^{-1}\ln T},\nn \\
			& \leq  1+ 4 \sqrt{15} d \beta^{-1}\ln T. \label{eq:brout0}
		\end{align}
		On the other hand, by H\"older's inequality, we have 
		\begin{align}
			|\inner{\g_t}{\p_t-\w_t}|  \leq \|\g_t\|_{\nabla^{-2}\Phi_{t-1}(\w_t)} \cdot \|\p_t-\w_t\|_{\nabla^2 \Phi_{t-1}(\w_t)} \leq 	 64 e^2\eta^{2}/c^3_T,   
		\end{align}
		where the last equality follows by Lemmas \ref{lem:close} and \ref{lem:handylem}. Combining this with the range assumptions on $\eta$ and $\beta$, we get 
		\begin{align}
			\sum_{t=1}^T  \inner{\g_t}{\p_t- \w_t}^2 \leq   \frac{64 e^2 \eta^{2}}{ c^3_T}\sum_{t=1}^T  |\inner{\g_t}{\p_t- \w_t}| \leq \frac{64 e^2 \eta^{2}}{ c^3_T}\left(  1+ 4\sqrt{15}  d \beta^{-1}\ln T\right), \label{eq:bezi0}
		\end{align}
		where the last inequality follows by \eqref{eq:brout0}. 
	\end{proof}
			\begin{proof}[{\bf Proof of Lemma \ref{lem:decomplemnew}}]
		Let $c_T\coloneqq 1-1/T$. We have 
		\begin{align}
			\sum_{t=1}^T	(D_{\Psi_t}(\u', \w_t) - D_{\Psi_{t-1}}(\u', \w_{t})) = \sum_{t=2}^T\sum_{i=1}^d\left(\frac{1}{\eta_{t,i}} - \frac{1}{\eta_{t-1,i}} \right) h\left(\frac{\bar u'_i}{\bar w_{t,i}}\right), \label{eq:downtown}
		\end{align}
		where $h(x) \coloneqq x - 1 - \ln x$ and $(\eta_{t,i})$ are defined in \eqref{eq:defs}. Fix $i\in [d]$ and let $\cT\coloneqq  \{t\colon \rho_{t,i} \neq \rho_{t-1,i}\}$. First, suppose that $\bar u_{t,i}\geq \frac{1}{2d}\wedge (c_T \bar u'_i)$, for all $t\in[T]$. In this case, we have $\max_{t\in[T]}\frac{\bar u'_i}{\bar u_{i,t}}\leq 2 d \bar u'_i + c_T^{-1}$. Therefore, by positivity of $h$ and the fact that $\eta_{t+1,i}\geq \eta_{t,i}$, we have \begin{align}\label{eq:again}
			\sum_{t=2}^T\left(\frac{1}{\eta_{t,i}} - \frac{1}{\eta_{t-1,i}} \right) h\left(\frac{\bar u'_i}{\bar w_{t,i}}\right) \leq 0 \leq c_T^{-1}+ 2 d \bar u'_i- \max_{t\in[T]} \frac{u_i'}{\bar u_{t,i}}.\end{align}
		Now, assume that $\bar u_{t,i}\leq \frac{1}{2d}\wedge (c_T \bar u'_i)$. Note that this implies that $\cT\neq \emptyset$ (due to $\bar u_{t,i}\leq 1/(2d)$ and the definition of $(\rho_{t,i})$). Let $\tau\coloneqq \max \cT$. For notational convenience, we let $$u \coloneqq \bar u'_i, \quad u_\tau\coloneqq \bar u_{\tau,i} ,  \quad w_\tau\coloneqq \bar w_{\tau,i}, \quad \text{and} \quad   \eta_\tau \coloneqq \eta_{\tau,i}.$$
		By definition of $(\eta_{t,i})$, we have $\eta_{\tau} = \eta \exp(-\log_T d  u_{\tau}) \leq  \eta \exp(\log_T T) = \eta e$. Thus, by the positivity of $h$, we have 
		\begin{align*}
			\sum_{t=2}^T\left( \frac{1}{\eta_{t,i}}-\frac{1}{\eta_{t-1,i}} \right)h\left(\frac{\bar u'_i}{\bar w_{t,i}}\right) & \leq  \left( \frac{1}{\eta_{\tau}} - \frac{1}{\eta_{\tau-1}} \right) h\left(\frac{u}{w_{\tau}}\right),\\
			& = \frac{1-e^{ \log_T({u_{\tau-1}}/{u_{\tau}})}    }{\eta_{\tau}}\cdot h\left(\frac{u}{w_{\tau}}\right), \\
			&	\leq   \frac{- \log_T\frac{u_{\tau-1}}{u_{\tau}}}{\eta_{\tau}}\cdot h\left(\frac{u}{w_{\tau}}\right),\\
			& \leq   -\frac{\log_2 \frac{u_{\tau-1}}{u_{\tau}}}{ e \eta \ln_2 T }\cdot h\left(\frac{u}{w_{\tau}}\right),\label{eq:neee}
		\end{align*}
		By definition of $(\u_{t,i})$, we have
		\begin{align}
			\frac{u_{\tau}}{w_\tau} = \frac{1/(Td) + (1-1/T)w_{\tau} }{w_\tau}\geq 1 - \frac{1}{T}=c_T.
		\end{align}
		Hence, $h(u/w_\tau)\geq h(c_T u/u_\tau)$, since $c_T u/u_\tau\geq 1$ by assumption and $h$ is positive and decreasing on $[1,+\infty[$. On the other hand, by construction of $(\rho_{t,i})$ and the definition of $\tau$, we have $u_{\tau}\leq u_{\tau-1}/2$, which implies that $-\log_2(u_{\tau-1}/u_{\tau})\leq -1$. Plugging these bounds into \eqref{eq:neee}, we get that 
		\begin{align}
			\sum_{t=2}^T\left(\frac{1}{\eta_{t,i}} - \frac{1}{\eta_{t-1,i}} \right) h\left(\frac{\bar u'_i}{\bar w_{t,i}}\right) &\leq - \frac{1}{e \eta \log_2 T} h\left(\frac{c_T u}{u_\tau}\right),\nn \\
			& = -\frac{1}{e \eta \log_2 T} \left( \frac{c_T u}{u_\tau} -1 - \ln \left(\frac{c_T u}{u_\tau}\right) \right),\\
			& =  -\frac{c_T}{e \eta \log_2 T}\max_{t\in[T]}\frac{\bar u'_i}{\bar u_{t,i}} + \mathcal{O}\left(\frac{\ln(dT\bar u'_i)}{\eta\ln T}\right),  \label{eq:ll}
		\end{align}
 Thus, by \eqref{eq:again}, \eqref{eq:ll}, and the fact that $\bar u_i''\leq 2 \bar u_i'$, for all $i\in[d]$, we get the desired result after summing over $i\in[d]$.
	\end{proof}

	\section{Proof of Theorem \ref{thm:mainthm} (Meta-Algorithm Regret)}
	\label{app:proofs}
	
	\begin{proof}[{\bf Proof of Proposition \ref{prop:mainprop}}]
		The proof of the proposition follows from an extension of the proof of \cite[Lemma 6]{zhang2019dual}.
		
		In this proof, we denote by $\cA_t$ the set of active experts at round $t$; that is, $\cA_t\coloneqq \{\A^{\beta,I}\colon I \in \cJ_t, \beta \in \cG\}$. For notational convenience, we will write $\sum_{\A^{\beta,I}\in \cA_t}$ to mean $\sum_{\beta\in \cG, I\in \cJ_t}$. By the assumption on the substitution function in \eqref{eq:sub}, we have 
		\begin{align}
			e^{-\eta f_t( \u_t)} = 	e^{-\eta f_t(\Upsilon(Q_{t-1}, U_t))}  \geq  \sum_{\A^{\beta,I}\in \cA_t} q^{\beta, I}_{t-1} e^{-\eta f_t( \w^{\beta, I}_t)}, \label{eq:mixable}
		\end{align}
		where $Q_{t-1} = (q_{t-1}^{\beta,I})_{\beta \in \cG, I \in \cJ_t}$, $U_t = (\u_t^{\beta,I})_{\beta \in \cG, I \in \cJ_t}$, and 
		\begin{align}
			q^{\alpha , J}_{t-1} =  \frac{e^{-\eta\cF_{t-1}^{\alpha,J}}}{\sum_{\A^{\beta, I}\in \cA_t} e^{-\eta\cF_{t-1}^{\beta,I}}}.
		\end{align}
		for any $J\in \cI|_{T}$ and $\alpha \in \cG$. Rearranging \eqref{eq:mixable} we get
		\begin{align}
			\sum_{\A^{\beta,I}\in \cA_t} e^{-\eta \cF^{\beta, I}_{t}} & = 	\sum_{\A^{\beta,I}\in \cA_t} e^{-\eta \cF^{\beta, I}_{t-1}} \cdot e^{\eta f_t( \u_t)- \eta f_t( \w^{\beta,I}_t)} \leq \sum_{\A^{\beta,I}\in \cA_t} e^{- \eta \cF^{\beta, I}_{t-1}}. \label{eqn:lem:meta:2}
		\end{align}
		Summing \eqref{eqn:lem:meta:2} over $t=1,\ldots,s$, we have
		\[
		\sum_{t=1}^s \sum_{\A^{\beta, I} \in \cA_t} \exp(-\eta \cF_{t}^{\beta, I})  \leq \sum_{t=1}^s  \sum_{\A^{\beta, I} \in \cA_t} \exp(-\eta \cF_{t-1}^{\beta, I})
		\]
		which can be rewritten as
		\[
		\begin{split}
			&\sum_{\A^{\beta, I} \in  \cA_s} \exp(-\eta \cF_{s}^{\beta, I}) +\sum_{t=1}^{s-1} \left(\sum_{\A^{\beta, I} \in  \cA_{t} \setminus \cA_{t+1}} \exp(-\eta \cF_{t}^{\beta, I}) +  \sum_{\A^{\beta, I} \in  \cA_{t} \cap \cA_{t+1}} \exp(-\eta \cF_{t}^{\beta, I}) \right)\\
			\leq &\sum_{\A^{\beta, I} \in  \cA_1} \exp(-\eta \cF_{0}^{\beta, I})+\sum_{t=2}^s \left( \sum_{\A^{\beta, I} \in \cA_t \setminus \cA_{t-1}} \exp(-\eta \cF_{t-1}^{\beta, I}) + \sum_{\A^{\beta, I} \in \cA_t \cap \cA_{t-1}} \exp(-\eta \cF_{t-1}^{\beta, I})\right)
		\end{split}
		\]
		implying
		\begin{equation}\label{eqn:lem:meta:3}
			\begin{split}
				&\sum_{\A^{\beta, I} \in  \cA_s} \exp(-\eta \cF_{s}^{\beta, I}) + \sum_{t=1}^{s-1}  \sum_{\A^{\beta, I} \in  \cA_{t} \setminus \cA_{t+1} } \exp(-\eta \cF_{t}^{\beta, I}) \\
				\leq  &\sum_{\A^{\beta, I} \in  \cA_1} \exp(-\eta \cF_{0}^{\beta, I}) + \sum_{t=2}^{s} \sum_{\A^{\beta, I} \in \cA_t \setminus \cA_{t-1}} \exp(-\eta \cF_{t-1}^{\beta, I}) \\
				=& \sum_{\A^{\beta, I} \in  \cA_1} \exp(0) + \sum_{t=2}^{s} \sum_{\A^{\beta, I} \in \cA_t \setminus \cA_{t-1}} \exp(0)\\
				=& |\cA_1| + \sum_{t=2}^{s}| \cA_t \setminus \cA_{t-1}|.
			\end{split}
		\end{equation}
		Note that $|\cA_1| + \sum_{t=2}^{s}| \cA_t \setminus \cA_{t-1}|$ is the total number of experts created until round $s$. From the structure of geometric covering intervals and the fact that $|\cG|\leq M$, we have
		\begin{equation}\label{eqn:lem:meta:4}
			|\cA_1| + \sum_{t=2}^{s}| \cA_t \setminus \cA_{t-1}| \leq M s \left(\lfloor \log_2 s \rfloor+1\right)  \leq M s^2.
		\end{equation}
		From \eqref{eqn:lem:meta:3} and \eqref{eqn:lem:meta:4}, we have
		\[
		\sum_{\A^{\beta, I} \in  \cA_s} \exp(-\eta \cF_{s}^{\beta, I}) + \sum_{t=1}^{s-1}  \sum_{\A^{\beta, I} \in  \cA_{t} \setminus \cA_{t+1} } \exp(-\eta \cF_{t}^{\beta, I}) \leq Ms^2.
		\]
		Thus, for any interval $I \in \cI$ and $s\in I$, we have
		\[
		\exp(-\eta \cF_{s}^{\beta,I}) = \exp\left( \eta \sum_{t\in I \cap[s]}(f_t(\u_t) - f_t(\u_t^{\beta,I})) \right)\leq Ms^2,
		\]
		which completes the proof.
	\end{proof}

	\subsection{Proof of Theorem \ref{thm:mainthm}}
	To prove Theorem \ref{thm:mainthm}, we define 
	\begin{align}
		\u_{*} \in \argmin_{\u\in  \cC_{d-1}} \sum_{t=1}^T \ell_t(\u), \quad \text{and} \quad  	\alpha^{\beta, I}_t \coloneqq 8^{-1}  \wedge \min_{s\in I \cap [t], i \in [d]}  ({\bar u^{\beta, I}_{t, i}}/({8\bar u''_{*,i}})).\label{eq:minimizer}
	\end{align}
	With this, we start by stating a regret guarantee against $\u_*$ for the subroutines of Algorithm~\ref{alg:DONSmeta}, which follows from Lemma \ref{lem:close} and the base algorithm's regret bound in Theorem \ref{lem:thekey0}.
	\begin{lemma}
		\label{lem:epochregret}
		Let $I\in \mathcal{I}$, $\beta \in \cG$, and $(\u^{\beta,I}_t)_{t\in I}$ be the outputs of the subroutine $\mathsf{A}^{\beta,I}$ within Algorithm~\ref{alg:DONSmeta} in response to Cover's losses $(\ell_t)$. Further, let $\u_*$ be as in \eqref{eq:minimizer}. If there exists $t\in I\cap [T-1]$ such that $\alpha^{\beta, I}_{t} \geq \beta >\alpha^{\beta,I}_{t+1},$ and $t> \min I$ then, 
		\begin{gather}
			\sum_{s\in I \cap [t]} (\ell_s( \u_s^{\beta,I}) - \ell_s(\u_*)) \leq O\left({d^2 \ln^4 T} \right) - \frac{204 d \ln^2 T}{\beta}. \label{eq:secondfirst} 
		\end{gather} 
		Otherwise, if $\alpha^{\beta, I}_{t} \geq \beta$ for all $t\in I$, then $
		\sum_{s\in I \cap [T]} (\ell_s( \u_s^{\beta,I}) - \ell_s(\u_*)) \leq O\left({d^2 \ln^4 T} \right) + \frac{34 d \ln T}{\beta}$. 
	\end{lemma}
For the proofs of Lemma \ref{lem:epochregret} and Theorem \ref{lem:thekey0}, we need the following Corollary to Lemma \ref{lem:close}:
\begin{corollary}
	\label{cor:thecor}
	In the setting of Lemma \ref{lem:close}, the iterate of $(\w_t)$ in Alg.~\ref{alg:DONSbase} and $(\p_t)$ in \eqref{eq:mirrordescent} satisfy:
	\begin{align}
		\forall i\in[d],\forall t\in[T],\quad 1 - {64 (e \eta)^2}\leq 	\frac{\bar p_{t,i}}{\bar w_{t,i}}\leq 1 + {64 (e \eta)^2}, \quad   \frac{3}{4} \leq 	\frac{\bar w_{t+1,i}}{\bar w_{t,i}}\leq \frac{5}{4}, \quad \text{and}\quad 	\frac{	\bar u_{t,i}}{\bar u_{t+1,i}}\leq \frac{4}{3}. \nn
	\end{align}	
\end{corollary}
The proofs of Lemma \ref{lem:epochregret} and Corollary \ref{cor:thecor} are in Appendix \ref{sec:epochregret}.

	\begin{proof}[{\bf Proof of Theorem \ref{thm:mainthm}}]
		 Recall that $\cG \coloneqq \{\frac{1}{d 2^{i+3}}\colon i\in [\ceil{\log_2 T}]\}$. We note that the outputs of $(\u^{\beta,I}_t)$ of the base algorithms are all in the set 
		\begin{align}
			\bar \cC_{d-1} \coloneqq \{\u \in \cC_{d-1}\colon \bar u_{i}\geq 1/(dT), \forall i\in[d] \}.
		\end{align}
This is because of the mixing with the uniform distribution on Line \ref{line:mix} of Algorithm \ref{alg:DONSbase}.	By \eqref{eq:mixture}, the outputs $(\u_t)$ of Algorithm \ref{alg:DONSmeta} are convex combinations of $(\u^{\beta,I}_t)$, and so $(\u_t)\subset \bar \cC_{d-1}$. This fact will be useful below.

		 For any $\beta \in \cG$, we define the set $$\cJ^\beta \coloneqq \{t\in [T]\colon \alpha^{\beta,I_0}_t<\beta , I_0 \coloneqq [T]\}.$$ Note that $I_0=[T]\in \cI|_{T}$, by definition, where $\cI|_{T}$ is the set of intervals indexing the base algorithms of Alg.~\ref{alg:DONSmeta}. Furthermore, for  $\beta' =\min \cG$, we have $\cJ^{\beta'}=\emptyset$ since $\beta'\leq \frac{1}{8 dT}\stackrel{(*)}{\leq} \alpha^{\beta',I}_t$ for any $I\in \cI|_{T}$ and $t\geq 1$, where $(*)$ follows by the fact that all the outputs $(\u^{\beta',I}_t)$ are in the set $\bar \cC_{d-1}$. If $\cJ_{\beta}=\emptyset$, for all $\beta\in \cG$, then by Proposition \ref{prop:mainprop} and the second inequality in Lemma \ref{lem:epochregret}, instantiated with $\beta_0 \coloneqq\max \cG = \frac{1}{2^4d}$ and $I_0 = [T]$, we get 
\begin{align}
	\sum_{t=1}^T ( \ell_t(\u) -\ell_t(\u_*)) \leq 	\sum_{t=1}^T ( \ell_t(\u^{\beta_0,I_0}_t) -\ell_t(\u_*)) + O(\ln T) \leq  O(\beta^{-1}_0 d\ln T) = O(d^2\ln T),
\end{align}
which implies the desired regret bound. Now, suppose that there exists $\beta \in \cG$ such that $\cJ^\beta\neq\emptyset$ and let $\beta_* \coloneqq \min \{\beta \in \cG\colon \cJ^\beta \neq \emptyset\}$ and $\tau \coloneqq \min \cJ^{\beta_*}$. By the definition of the base Algorithm \ref{alg:DONSbase}, we have $\u_{1}^{\beta_*, I_0}=\mathbf{1}/d$, and thus, 
\begin{align}
	2 \beta_0 =\frac{1}{2^3d}\leq \alpha_{1}^{\beta_*, I_0} \quad \text{and} 	\quad  \beta_*   \leq  \beta_0 \leq \alpha_{2}^{\beta_*, I_0},
\end{align}
where the last inequality follows by the fact that ${2 \beta_0}\leq \alpha_{1}^{\beta_*, I_0}$ and the bound on $\bar u_{1,i}/\bar u_{2,i}$ for $i\in[d]$ from Corollary \ref{cor:thecor}. Therefore, by definition of $\tau$ we must have that $\tau> 2$ and $\alpha_{\tau-1}^{\beta_*, I_0}\geq \beta_* > \alpha_{\tau}^{\beta_*,I_0}$. Thus, by invoking Proposition \ref{prop:mainprop} and Lemma~\ref{lem:epochregret} (in particular \eqref{eq:secondfirst}), we get
	\begin{align}
	\sum_{t=1}^{\tau-1}(\ell_t( \u_t) - \ell_t(\u_*)) &\leq 	\sum_{t=1}^{\tau-1}(\ell_t(\u^{\beta_*, I_0}_t) - \ell_t(\u_*)) + O(\ln T),\nn \\
	& \leq O \left(d^2 \ln^4 T \right) - \frac{204 d \ln^2 T}{\beta_*}. \label{eq:grenz}
\end{align} 
Now, let $\beta'\coloneqq \beta_*/2$ and $\cS = \mathcal{S}([\tau+1,T])\subset \cI$ denote any disjoint partition of $[\tau+1,T]$ that contains at most $2\log_2 T$ elements in $\cI$; this is guaranteed to exists by \cite[Lemma~5]{daniely2015strongly}.  The fact that $\mathcal{J}^{\beta'} = \emptyset$ together with Proposition \ref{prop:mainprop} and the second inequality in Lemma~\ref{lem:epochregret} implies
\begin{align}
	\sum_{t=\tau+1}^{T} (\ell_t( \u_t) - \ell_t(\u_*)) & \leq \sum_{I \in \cS}\sum_{t\in I}  (\ell_t( \u_t^{\beta', I}) - \ell_t(\u_*)) + O(|\cS|\ln T),\nn \\& \leq  O \left( |\cS| d^2 \ln^4 T \right) + \frac{34|\cS| d \ln T}{\beta'},	\nn \\
	& \leq O \left(d^2 \ln^5 T \right) + \frac{204 d \ln^2 T}{\beta_*},
\end{align}
where the last inequality, we used the fact that $\beta'=\beta_*/2$ and $|\cS|\leq 2 \ln_2 T\leq 3 \ln T$. Combining this with \eqref{eq:grenz} and using the fact that $\ell_\tau(\u_\tau)-\ell_\tau(\u_*)\leq \ln (dT)$ (since $\u_\tau\in \bar \cC_{d-1}$) completes the proof.
	\end{proof}

		\subsection{Proofs of Lemma \ref{lem:epochregret} and Corollary \ref{cor:thecor}}
	\label{sec:epochregret}

			\begin{proof}[{\bf Proof of Corollary \ref{cor:thecor}}]
		We start by the first inequality. By Lemma \ref{lem:close}, we have 
		\begin{align} 
			\|\w_t - \p_t\|_{\nabla^2 F_t(\w_t)}\leq 64 (e\eta)^{3/2} .\label{eq:field}
		\end{align} 
		Note also that for any $i \in[d-1]$, we have $\nabla^2 F_t(\w_t) \succeq \nabla^2 \Psi_t(\w_t) \succeq \e_i \e_i^\top/(\eta e w_{t,i}^{2})$, since $\eta_{t,i} \in [\eta ,\eta e]$. Therefore, \eqref{eq:field} implies that $ (w_{t,i}- p_{t,i})^2/(\eta e w_{t,i}^2) \leq 64^2 (e\eta)^3$. Rearranging this implies the bounds on ${\bar p_{t,i}}/{\bar w_{t,i}}$ for the case where $i \in[d-1]$. Also, we have that $\nabla^2 F_t(\w_t)\succeq  \nabla^2 \Psi_t(\w_t)  \succeq   \mathbf{1}\mathbf{1}^\top/(\eta e \bar w_{t,d}^{2})$. Therefore, \eqref{eq:field} implies that
		\begin{align}
			64^2 (e\eta)^3  \geq  (\inner{\w_{t}}{\mathbf{1}}- \inner{\p_t}{\mathbf{1}})^2/( \eta e \bar w_{t,d}^2) = (\bar w_{t,d} - \bar p_{t,d})^2/(\eta e \bar w_{t,d}^2). \nn
		\end{align}
		Rearranging this inequality yields the desired bounds on ${\bar p_{t,i}}/{\bar w_{t,i}}$ for $i=d$. We now bound $\bar w_{t+1,i}/\bar w_{t,i}$ for $i\in[d]$. By definition of $\w_{t+1}$, we have 
		\begin{align}
			\|\w_{t+1} - \w_t\|_{\nabla^2 \Psi_t(\w_t)} &\leq 	\|\w_{t+1} - \w_t\|_{(\nabla^2 \Psi_t(\w_t) + V_t)},\nn \\ & = \left\|\frac{(\nabla^2\Psi_t(\w_t)+V_t)^{-1}\bm{\nabla}_t}{1+4\sqrt{e\eta} \|\bm{\nabla}_t\|_{(\nabla^2\Psi_t( \w_t)+V_t)^{-1}}}  \right\|_{(\nabla^2 \Psi_t(\w_t) + V_t)}\nn \\ & = \frac{\|\bm{\nabla}_t\|_{(\nabla^2\Psi_t( \w_t)+V_t)^{-1}}}{1+4\sqrt{e\eta} \|\bm{\nabla}_t\|_{(\nabla^2\Psi_t( \w_t)+V_t)^{-1}}} \leq\frac{1}{4 \sqrt{e\eta}}.
		\end{align}
		Thus, since $\nabla^2 \Psi_t(\w_t)\succeq \e_i \e_i^\top/(\eta e \bar w_{t,i}^{2})$, for all $i \in [d]$, we get that (as we did above)
		\begin{align}
			\frac{(\bar w_{t+1,i} - \bar w_{t,i})^2}{\eta e\bar w^2_{t,i}} \leq \frac{1}{16 \eta e}, \quad \forall i \in[d].
		\end{align}
		Thus, after rearranging, we obtain the desired bounds on $\bar w_{t+1,i}/\bar w_{t,i}$. Finally, we have 
		\begin{align}
			\frac{	\bar u_{t,i}}{\bar u_{t+1,i}} = \frac{(1-1/T)	\bar w_{t,i} + 1/(dT) }{(1-1/T)\bar w_{t+1,i}  +1/(dT)} \leq 1 \vee \frac{\bar w_{t,i}}{\bar w_{t+1,i}} \leq \frac{4}{3}, \label{eq:charm}
		\end{align}
		where the last inequality follows by the fact that ${3}/{4} \leq 	{\bar w_{t+1,i}}/{\bar w_{t,i}}\leq {5}/{4}$.
	\end{proof}
	
			\begin{proof}[{\bf Proof of Lemma \ref{lem:epochregret}}]
		Let $\u_*$ be as in \eqref{eq:minimizer}. We first show \eqref{eq:secondfirst}. Since $\beta\leq \alpha_{t}^{\beta, I}$, Lemma \ref{lem:handy} implies that $$\beta \leq 8^{-1} \wedge |8\inner{\g_s^{\beta, I}}{\u''_* -\u_{s}^{\beta, I}}|^{-1},$$ for all $s\in I \cap [t]$, where $\g_{s}^{\beta,I} \coloneqq \nabla \ell_s(\u_s^{\beta,I})$. Therefore, by the assumption that $t> \min I$ and Theorem \ref{lem:thekey0}, we have
		\begin{align}
			\sum_{s\in I \cap [t]} (\ell_s(\u^{\beta, I}_s) -  \ell_s(\u_{*}) \leq	O\left(\frac{d \ln T}{\eta} \right)   + \frac{34 d \ln T}{\beta} - \frac{a_t}{32 \eta \ln T}, \label{eq:hold for t-1}
		\end{align}
		where $a_t \coloneqq \max_{s\in I \cap [t]} \sum_{i=1}^d  {\bar u''_{*,i}}/{\bar u_{s,i}}$. Now the fact that $\beta >\alpha_{t+1}^{\beta, I}$ implies
		\begin{align}
			\frac{1}{\beta}
			&\leq  8 \vee \max_{s\in I \cap [t+1],i\in[d]}\frac{8\bar u''_{*,i}}{\bar u_{s,i}} \leq  8 a_{t+1} + 8 \leq 32 a_{t}/3 + 8, \label{eq:a_t-1_is_large}
		\end{align}
		where the last step follows by the bound on $\bar u_{t,i}/\bar u_{t+1,i}$ from Corollary \ref{cor:thecor}. Further, combining this with \eqref{eq:hold for t-1}, we get that 
		\begin{align*}
			\sum_{s\in I \cap [t]}^t (\ell_s(\u^{\beta, I}_s) - \ell_s(\u_*))
			&\leq \mathcal{O}\left(\frac{ d\ln T}{\eta}\right) + \frac{34 d\ln T}{\beta}-\frac{1}{32\eta \ln T}\left(\frac{3}{32\beta}-\frac{3}{4}\right), 
			\tag{by~\eqref{eq:hold for t-1} and~\eqref{eq:a_t-1_is_large}}\\
			&\leq \mathcal{O}\left( d^2\ln^4 T\right) - \frac{204 d\ln^2 T}{\beta}. 
		\end{align*}	
		where we used that $\eta  = 1/(286^2 d\ln^3 T)$. When, $\beta \leq \alpha_t^{\beta,I}$, for all $t\in I$, we can simply invoke \eqref{eq:hold for t-1} and discard the negative term to obtain the second claimed bound of the lemma.
	\end{proof}

	\section{Why the \adabar{} Restarts Work}
\label{sec:sketch}
Before discussing the \adabar{} restarts, we first give some more details on the regret bound of the base algorithm. In particular, we highlight that the RHS of \eqref{eq:cancelling} (which is a bound on the sum of divergences in the regret bound of $\mathsf{BARRONS}$) can cancel the bound $O(\beta^{-1} d\ln T)$ on the stability term as long as $\beta \geq  \alpha_T(\u)/2$ and $\eta \leq O(1/(d\ln T))$. This is because $\alpha_T(\u) \geq \frac{1}{8} \wedge \min_{i\in[d],t\in[T]}  \frac{p_{t,i}}{8u_{t,i}}$. In fact, by choosing $\eta$ small enough and as long as $\alpha_T(\u)/2\leq \beta \leq \alpha_T(\u)$ on can ensure that what remains in the regret bound of $\mathsf{BARRONS}$ is 
\begin{align}
	R_T(\u)\leq \wtilde O(d^2) - \frac{C d  \ln^2 T }{\beta},\label{eq:remains}
\end{align}
for some $C>0$. Before discussing how one might ensure that $\alpha_T(\u)/2\leq \beta \leq \alpha_T(\u)$, we reiterate that what enables the cancellation of the term $O(\beta^{-1} d\ln T)$ is the sum $\sum_{t=1}^T (D_{\Phi_t}(\u,\p_t)- D_{\Phi_{t}}(\u, \p_{t+1}))$ in \eqref{eq:decomp2} that comes out of the mirror descent analysis. This idea of canceling terms thanks to negative Bregman divergence terms originated from the work by \cite{agarwal2017corralling} on combining Bandit algorithms. \cite{foster2020adapting} recently showed that on can use FTRL with slightly modified gradients instead of mirror descent to achieve the same goal in a Bandit setting. The approach of the latter fails to lead to an efficient algorithm in our setting since it is not possible to show that the Newton decrement of the damped Newton step is small on rounds where the iterates $(u_{t,i})$ reach new minima. 

As we already discussed in \S\ref{sec:barrons}, since the sequence of returns $(\r_t)$ is not known up-front, it is not possible for any algorithm to pick $\beta$ so that the condition $\alpha_T(\u)/2\leq \beta \leq \alpha_T(\u)$ is always satisfied. Aggregating multiple instances of $\mathsf{BARRONS}$ with different $\beta$'s also fails since $\alpha_T(\u)$ depends on the outputs of the algorithm; and so changing $\beta$ changes the target $\alpha_T(\u)$ for the base algorithm (see also discussion in \cite{luo2018efficient}). Instead of aggregating base algorithms, the approach taken by \cite{luo2018efficient} consists of restarting the base algorithm on round $t$ if the current estimate for $\beta$ satisfies $\beta > \alpha_t(\u_t)$, where $\u_t$ is the regularized leader:
\begin{align}
	\u_t \in \argmin_{\u \in \bar \Delta_d}  \sum_{i=1}^d \frac{-\ln u_{i}}{\eta_{t,i}}  + \sum_{s=\tau}^t \ell_s(\u). \label{eq:FTRL0}
\end{align}
and $\tau$ is the round where the current instance of the base algorithm was initialized. By stability of the iterates $(\u_t)$ and $(\p_t)$ it is possible to show that $\alpha_{t-1}(\u_{t-1}) \geq  \beta \geq  \alpha_{t-1}(\u_{t-1})/2$, and so invoking the regret bound of the base algorithm for the current epoch yields $\sum_{s=\tau}^{t-1} (\ell_t(\p_t)-\ell_t(\u_t)) \leq O(d^2 \ln^2 T) - C \beta^{-1} d\ln T$. When there are no restarts, the regret bound of the algorithm is simply $\sum_{s=\tau}^{t-1} (\ell_t(\p_t)-\ell_t(\u_t)) \leq O(d^2 \ln T) + C \beta^{-1} d\ln^2 T$. Starting with $\beta=\beta_0$ and halving $\beta$ every time there is a restart and letting $\tau_i$ [resp.~$\beta_i$] be the start round [resp.~the $\beta$] of epoch $i$ and $M$ be the total number of epochs, the regret of the meta-algorithm is bound by 
\begin{align}
	R_T(\u) & \leq \sum_{i=1}^M|\ell_{\tau_i}(\p_{\tau_{i}}) -\ell_{\tau_i}(\u)|+\sum_{i=1}^M\sum_{t=\tau_i+1}^{\tau_{i+1}-1}(\ell_{t}(\p_t) - \ell_t(\u)),\nn \\ \displaybreak[1]
	& \leq \sum_{i =1}^M|\ell_{\tau_i}(\p_{\tau_{i}}) -\ell_{\tau_i}(\u)|+\sum_{i=1}^M\sum_{t=\tau_i+1}^{\tau_{i+1}-1}(\ell_t(\p_t) - \ell_t(\u_t)),\nn \\\displaybreak[1]
	& \label{eq:newoneone} \leq \sum_{i=1}^{M-1} \wtilde O(1) + \sum_{i=1}^{M-1} \left(\wtilde O(d^2) - \frac{Cd \ln^2 T}{\beta_i}\right) + \wtilde O(d^2)  +  \frac{Cd \ln^2 T}{\beta_M}, \\\displaybreak[1]
	& =  \wtilde O(d^2)  - \sum_{i=1}^{M-1}\frac{C   d \ln^2 T}{2^{1-i} \beta_0} +\frac{ C d  \ln^2 T}{2^{1-N}\beta_0} =  \wtilde O(d^2) + \frac{C d \ln^2 T}{\beta_0}, \label{eq:desired}
\end{align}
where \eqref{eq:newoneone} follows from the regret bound of the base algorithm and the fact that $\p_t,\u \in \bar \Delta_d$, for all $t$ (the latter ensures that $\ell_t(\p_t)- \ell_t(\u) \leq \wtilde O(1)$), and the last inequality follows by the fact that $M\leq O(\ln T)$ since $\inf_{\u\in \bar \Delta_d}\alpha_T(\u)\geq 1/(dT)$ (also because $(\p_t)\subset \bar \Delta_d$). By choosing $\beta_0 = \Omega(1)$, \eqref{eq:desired} leads to the desired regret bound.

\end{document}

%% file: tex_header.tex
\usepackage{amsmath}

\usepackage{bbm}
\usepackage{bm}
\usepackage{multirow}
\usepackage{enumitem}
\usepackage{marginnote}
\usepackage{graphicx}
\usepackage{mdframed}
\usepackage{mathcomp}
\usepackage{stmaryrd}
\usepackage{mathtools}
\usepackage{xr}
\usepackage{nicefrac}       
\usepackage{hyperref}       

\usepackage{microtype}      

\usepackage{amsthm}
\usepackage{amssymb}
\usepackage{amsfonts}       
\newtheorem{theorem}{Theorem}
\newtheorem{lemma}[theorem]{Lemma}
\newtheorem{corollary}[theorem]{Corollary}
\newtheorem{proposition}[theorem]{Proposition}
\newtheorem{definition}[theorem]{Definition}
\usepackage{xcolor}
\usepackage[letterpaper, left=1in, right=1in, top=1in, bottom=1in]{geometry}

\newcommand{\cR}{\mathcal{R}}
\newcommand{\cC}{\mathcal{C}}

\renewcommand{\v}{\bm{v}}

\newcommand{\cU}{\mathcal{U}}
\newcommand{\cW}{\mathcal{W}}

\newcommand{\cT}{\mathscr{T}}

\newcommand{\cG}{\mathcal{G}}

\newcommand{\wtilde}{\widetilde}

\newcommand{\loss}{\ell}
\DeclareBoldMathCommand{\vloss}{\loss}
\DeclareBoldMathCommand{\grad}{g}
\DeclareBoldMathCommand{\fakegrad}{\mathring{\bm{g}}}
\DeclareBoldMathCommand{\e}{e}
\DeclareBoldMathCommand{\p}{p}
\DeclareBoldMathCommand{\u}{u}
\DeclareBoldMathCommand{\w}{w}
\DeclareBoldMathCommand{\x}{x}
\DeclareBoldMathCommand{\l}{l}
\DeclareBoldMathCommand{\vzero}{0}
\let\top\intercal

\newcommand{\reals}{\mathbb{R}}

\DeclareMathOperator{\E}{\mathbb{E}}

\renewcommand{\x}{\bm{x}}   

\newcommand{\y}{\bm{y}} 
\newcommand{\g}{\bm{g}}

\renewcommand{\r}{\bm{r}}

\DeclareMathOperator*{\argmin}{argmin}

\newcommand{\adabar}{$\mathsf{Ada}$-$\mathsf{BARRONS}$}
\newcommand{\G}{\mathbf{G}}
\newcommand{\cK}{\mathcal{K}}
\newcommand{\bnabla}{\bm{\nabla}}
\newcommand{\inte}{\operatorname{int}}
\newcommand{\dom}{\operatorname{dom}}

\newcommand*{\ceil}[1]{\lceil{#1}\rceil}

\newcommand{\cI}{\mathcal{I}}
\newcommand{\cA}{\mathcal{A}}
\renewcommand{\cT}{\mathcal{T}}
\newcommand{\cF}{\mathcal{F}}
\newcommand{\cL}{\mathcal{L}}
\newcommand{\cJ}{\mathcal{J}}
\newcommand{\A}{\mathsf{A}}

\newcommand{\nn}{\nonumber}

\newcommand{\inner}[2]{\langle#1,#2\rangle}

\newcommand{\norm}[1]{\left\lVert#1\right\rVert}

\newmdtheoremenv{condition}{Condition}

\newcommand{\vertiii}[1]{{\left\vert\kern-0.25ex\left\vert\kern-0.25ex\left\vert #1 
    \right\vert\kern-0.25ex\right\vert\kern-0.25ex\right\vert}}

\makeatletter
\newcommand*\bcdot{\mathpalette\bigcdot@{.5}}
\newcommand*\bigcdot@[2]{\mathbin{\vcenter{\hbox{\scalebox{#2}{$\m@th#1\bullet$}}}}}

\newlength\myindent